\documentclass[acmsmall]{acmart}

\usepackage{url}
\PassOptionsToPackage{hyphens}{url}
\usepackage{makecell}
\usepackage{multirow}
\usepackage{amsmath}
\usepackage{amsthm}
\usepackage{graphicx}
\graphicspath{ {images/} }
\usepackage[font=normalsize]{subfig}
\usepackage{algorithm}
\usepackage{algpseudocode}

\algnewcommand{\LineComment}[1]{\State \(\triangleright\) #1}

\theoremstyle{definition}

\newtheorem{theorem}{Theorem}

\newtheorem{lemma}{Lemma}

\AtBeginDocument{%
  \providecommand\BibTeX{{%
    \normalfont B\kern-0.5em{\scshape i\kern-0.25em b}\kern-0.8em\TeX}}}

\setcopyright{acmlicensed}
\copyrightyear{2023}
\acmYear{2023}
\acmDOI{10.1145/3633286}

\acmJournal{TIST}
\acmVolume{1}
\acmNumber{1}
\acmArticle{1}
\acmMonth{11}

\begin{document}

\title{Hierarchical Pruning of Deep Ensembles with Focal Diversity}

\author{Yanzhao Wu}
\email{yawu@fiu.edu}
\orcid{0000-0001-8761-5486}
\affiliation{%
  \institution{Florida International University}
  \streetaddress{11200 SW 8TH ST}
  \city{Miami}
  \state{Florida}
  \country{USA}
  \postcode{33199}
}
\additionalaffiliation{%
  \institution{Georgia Institute of Technology}
  \streetaddress{266 Ferst Drive}
  \city{Atlanta}
  \state{Georgia}
  \country{USA}
  \postcode{30332}
}

\author{Ka-Ho Chow}
\email{khchow@gatech.edu}
\author{Wenqi Wei}
\email{wenqiwei@gatech.edu}
\author{Ling Liu}
\email{lingliu@cc.gatech.edu}
\affiliation{%
  \institution{Georgia Institute of Technology}
  \streetaddress{266 Ferst Drive}
  \city{Atlanta}
  \state{Georgia}
  \country{USA}
  \postcode{30332}
}

\begin{abstract}
Deep neural network ensembles combine the wisdom of multiple deep neural networks to improve the generalizability and robustness over individual networks. It has gained increasing popularity to study and apply deep ensemble techniques in the deep learning community.
Some mission-critical applications utilize a large number of deep neural networks to form deep ensembles to achieve desired accuracy and resilience, which introduces high time and space costs for ensemble execution.
However, it still remains a critical challenge whether a small subset of the entire deep ensemble can achieve the same or better generalizability and how to effectively identify these small deep ensembles for improving the space and time efficiency of ensemble execution.
This paper presents a novel deep ensemble pruning approach, which can efficiently identify smaller deep ensembles and provide higher ensemble accuracy than the entire deep ensemble of a large number of member networks. 
Our hierarchical ensemble pruning approach (HQ) leverages three novel ensemble pruning techniques.
First, we show that the focal ensemble diversity metrics can accurately capture the complementary capacity of the member networks of an ensemble team, which can guide ensemble pruning.
Second, we design a focal ensemble diversity based hierarchical pruning approach, which will iteratively find high quality deep ensembles with low cost and high accuracy.
Third, we develop a focal diversity consensus method to integrate multiple focal diversity metrics to refine ensemble pruning results, where smaller deep ensembles can be effectively identified to offer high accuracy, high robustness and high ensemble execution efficiency.
Evaluated using popular benchmark datasets, we demonstrate that the proposed hierarchical ensemble pruning approach can effectively identify high quality deep ensembles with better classification generalizability while being more time and space efficient in ensemble decision making.
We have released the source codes on GitHub at \url{https://github.com/git-disl/HQ-Ensemble}.
\end{abstract}



\maketitle

\section{Introduction}
It has become an attractive learning technique to leverage deep neural network (DNN) ensembles to improve the overall generalizability and robustness of many deep learning systems. Some mission-critical applications often require a large number of deep neural networks to achieve the target accuracy and robustness, which entails high space and time costs for ensemble execution.
Recent studies have revealed that deep neural network ensembles with highly diverse member networks tend to have high failure independence, which is critical for enhancing overall ensemble predictive performance, including ensemble accuracy and robustness under adverse situations~\cite{ensemble-mass,ensemble-tdsc,EnsembleBenchCogMI,EnsembleBenchCVPR,heterobust}.
However, the member networks in a large deep ensemble team may not have high ensemble diversity and failure independence to complement each other, which will result in sub-optimal ensemble prediction performance and high ensemble execution cost in practice~\cite{deepensembles,ensemble-mass,ensemble-icnc,EnsembleBenchCogMI,EnsembleBenchCVPR,EnsembleBenchICDM}. For a deep ensemble team of a large size, such as 10, it is often not only possible to identify substantially smaller deep ensembles (e.g., 3$\sim$5 member networks) with the same or improved ensemble accuracy but also beneficial to reduce the ensemble execution cost~\cite{effectivepruning,ensemblepruninganalysis,EnsembleBenchCVPR,EnsembleBenchCogMI,EnsembleBenchICDM}.
This motivates us to propose an efficient hierarchical ensemble pruning approach, coined as HQ. By leveraging our focal diversity metrics, the HQ pruning method can effectively identify high quality deep ensembles with small sizes and high ensemble diversity, which not only improves ensemble accuracy over the large entire ensembles but also significantly reduces the space and time cost of ensemble execution.

\subsection{Related Work and Problem Statement}
Most of the existing ensemble learning studies can be summarized into three broad categories.
The first category builds ensemble teams by training a set of member models, represented by bagging~\cite{bagging}, boosting~\cite{boosting} and random forests~\cite{randomforest}. In practice, it has led to large ensemble teams of tens to hundreds of member models, such as the commonly-used random forests.
The second category is to leverage diversity measurements to compare and select high quality ensemble teams whose member models can complement each other to improve the ensemble predictive performance without requiring model training.
Ensemble diversity can be evaluated using pairwise or non-pairwise diversity measures, represented by Cohen's Kappa (CK)~\cite{cohenskappa} and Binary Disagreement (BD)~\cite{binarydisagreement} for pairwise metrics and Kohavi-Wolpert Variance (KW)~\cite{kwvariance,diversityaccuracy} and Generalized Diversity (GD)~\cite{generalizeddiversity} for non-pairwise metrics.
Early studies~\cite{diversityaccuracy,analysis-diversity-measures} have reported that it is challenging to use these diversity metrics for evaluating the performance quality of ensemble teams. Some recent studies~\cite{deepensembles,ensemble-mass,EnsembleBenchCVPR,EnsembleBenchCogMI} discussed inherent problems of using these diversity metrics to measure ensemble diversity in terms of failure independence and provided guidelines for improving the ensemble diversity measurement. This paper contributes to this second category by leveraging focal ensemble diversity metrics for effective pruning of deep ensembles.
The third category covers the ensemble consensus methods for aggregating member model predictions and producing the ensemble prediction, such as soft voting (model averaging), majority voting and plurality voting~\cite{ju2017relative} or learn to combine algorithms~\cite{learn-to-rank,EnsembleBenchCogMI}.

\begin{table*}[h!]
\centering
\caption{Summary of Three Categories of Related Studies}
\label{table:related-work-three-category}
\small
\scalebox{0.79}{
\begin{tabular}{|c|c|c|}
\hline
Category             & Description                                             & Representative methods                  \\ \hline
1. Ensemble training & \makecell{Training a set of member models\\to build ensemble teams} & \makecell{Bagging~\cite{bagging}, boosting~\cite{boosting},\\random forests~\cite{randomforest}, etc.} \\ \hline
\makecell{2. Ensemble diversity powered\\ensemble pruning} &
  \makecell{Leveraging ensemble diversity measurements\\to compare and select high-quality ensembles} &
  \makecell{CK pruning~\cite{cohenskappa}, BD pruning~\cite{binarydisagreement},\\GD pruning~\cite{generalizeddiversity}, etc.} \\ \hline
3. Ensemble consensus methods &
  \makecell{Aggregating member model predictions\\to produce the ensemble prediction} &
  \makecell{Soft voting, majority voting,\\plurality voting, etc.} \\
\hline
\end{tabular}
} 
\end{table*}

\begin{table*}[h!]
\centering
\caption{Comparison of Hierarchical Pruning and Existing Studies}
\label{table:ensemble-pruning-method-comparison}
\small
\scalebox{0.79}{
\begin{tabular}{|c|c|c|c|c|}
\hline
\multirow{2}{*}{Method} &
  \multirow{2}{*}{\makecell{Model Types}} &
  \multicolumn{3}{c|}{Diversity Measurement Principles} \\ \cline{3-5}
 &
   &
Diversity Comparison &
Samples for Measurements&
Diversity Calculation \\ \hline
\makecell{Early studies\\on diversity based\\ensemble pruning\\before 2015\\\cite{ensemble-selection-model-library,ensemble-diversity-thinning,ensemble-pruning-primer,ensemblepruninganalysis,effectivepruning}} &
  \makecell{Ensembles\\of traditional\\ML models} &
  \makecell{Compare\\all ensembles\\of different sizes} &
  \makecell{Random samples\\from the validation set} &
  \makecell{Directly calculated \\on random samples} \\ \hline
\makecell{\makecell{Focal diversity\\powered hierarchical\\ensemble pruning}} &
  \makecell{Ensembles of \\(1) DNNs and \\(2) Traditional\\ML models} &
  \makecell{Compare\\the ensembles\\of the same size $S$} &
  \makecell{Random negative samples\\from a focal model $F_f$} &
  \makecell{Obtain focal negative correlations\\for each focal (member) model $F_f$\\in an ensemble team of size $S$\\and then perform an average of $S$\\focal negative correlation scores\\to calculate the focal diversity score} \\ \hline
\end{tabular}
} 
\end{table*}

We summarize the three categories of related studies in Table~\ref{table:related-work-three-category}.
These three categories are complementary. To develop an ensemble team for improving prediction performance, we can employ one of the two ways to obtain an ensemble of member models: (1) training multiple models together using an ensemble learning algorithm, such as bagging, boosting, or random forest; and (2) training multiple models independently in parallel and employ a voting method to produce ensemble consensus based predictions. 
Most of the early studies before 2015 in the first two categories~\cite{ensemble-selection-model-library,ensemble-diversity-thinning,ensemble-pruning-primer,ensemblepruninganalysis,effectivepruning} focused on ensemble pruning for traditional machine learning models. Until recently, we have observed several research endeavors on deep neural network ensembles, most of which centered on training multiple networks jointly, such as diversity based weighted kernels~\cite{kaho-sigkdd2021,object-max-ensemble-pruning,learning-to-diversify} and leveraging deep ensembles to strengthen the robustness and resilience of a single deep neural network under adverse situations~\cite{ensemble-mass,ensemble-icnc,ensemble-tdsc,ensemble-bigdata,kaho-sigkdd2021,EnsembleBenchCVPR,heterobust}. However, to the best of our knowledge, very few studies have brought forward solutions to efficient deep ensemble pruning to improve prediction performance and reduce ensemble execution cost.

We compare our focal diversity powered hierarchical ensemble pruning approach and early studies using ensemble diversity for ensemble pruning in Table~\ref{table:ensemble-pruning-method-comparison}. Our focal diversity promotes fair comparison of ensemble diversity among the ensembles of the same size $S$ and leverages the focal model concept, focal negative correlation, and averaging of multiple focal negative correlation scores to obtain accurate ensemble diversity measurements.
We will provide a detailed description of focal diversity powered hierarchical pruning in Section~\ref{section:focal-diversity-hierarchical pruning} and demonstrate that our hierarchical pruning approach can be effectively applied to both deep neural network ensembles and ensembles of traditional machine learning models in Section~\ref{section:experimental-evaluation}.

\subsection{Scope and Contribution}
This paper presents a holistic approach to efficient deep ensemble pruning. Given a large deep ensemble of $M$ member networks, we propose a hierarchical ensemble pruning framework, denoted as HQ, to efficiently identify high quality ensembles with lower cost and higher accuracy than the entire ensemble of $M$ networks. Our HQ framework combines three novel ensemble pruning techniques.
{\it First,} we leverage focal ensemble diversity metrics~\cite{EnsembleBenchCVPR} to measure the failure independence among member networks of a deep ensemble. The higher focal diversity score indicates the higher level of failure independence among member networks of a deep ensemble, that is higher complementary capacity for improving ensemble predictions. Our focal diversity metrics can precisely capture such correlations among member networks of an ensemble team, which can be used to effectively guide ensemble pruning.
{\it Second,} we present a novel hierarchical ensemble pruning method powered by focal diversity metrics. Our hierarchical pruning approach iteratively identifies subsets of member networks with low diversity, which tend to make similar prediction errors, and then prunes them out from the entire ensemble team.
{\it Third,} we perform focal diversity consensus voting to combine multiple focal diversity metrics for ensemble pruning, which further refines the hierarchical ensemble pruning results by a single focal diversity metric.
Comprehensive experiments are performed on four popular benchmark datasets, CIFAR-10~\cite{cifar10-100}, ImageNet~\cite{ILSVRC}, Cora~\cite{cora} and MNIST~\cite{mnistlenet}. 
The experimental results demonstrate that our focal diversity based hierarchical pruning approach is effective in identifying high quality deep ensembles with significantly smaller sizes and better ensemble accuracy than the entire deep ensemble team.

\section{Ensemble Pruning with Focal Diversity} \label{section:ensemble-pruning-intro}
Given an entire deep ensemble with a large size $M$, it consists of $M$ individual member networks ($F_i, i\in\{0, 1, ..., M-1\}$) that are trained for a specific learning task and dataset. We denote the set of all possible sub-ensembles as $EnsSet$, which are composed of any subset of these $M$ individual networks. 
For a specific team size $S$, let $EnsSet(S)$ denote the set of all possible sub-ensembles of size $S$ in $EnsSet$. The cardinality of $EnsSet(S)$ is calculated based on the selection of $S$ networks from all $M$ base networks, that is $|EnsSet(S)| = \binom{M}{S}$.
Therefore, the total number of candidate sub-ensembles for $S=2,\dots, M-1$ is $|EnsSet|=\sum_{S=2}^{M-1} |EnsSet(S)| = \binom{M}{2} + \binom{M}{3} + ... + \binom{M}{M-1} = 2^M - (2+M)$, which grows exponentially with $M$. For example, when $M=3,~5,~10,~20$, we have $|EnsSet|=3,~25,~1012,~1048554$ respectively. 
It may not be feasible to perform the exhaustive search of all possible ensembles in $EnsSet$ with a large $M$. Hence, it is critical to develop efficient ensemble pruning methods for examining the sub-ensembles from the candidate set $EnsSet$, removing these sub-ensembles with low diversity and obtaining the set $GEnsSet$ of high quality sub-ensembles with lower cost and better ensemble accuracy than the entire deep ensemble of $M$ member networks.

\begin{table*}[h!]
\centering
\caption{Example Deep Ensembles for CIFAR-10 and ImageNet}
\label{table:example-ensembles-cifar10-imagenet}
\small
\scalebox{0.9}{
\begin{tabular}{|c|c|c|c|c|c|c|c|c|c|c|}
\hline
Dataset & \multicolumn{5}{c|}{CIFAR-10} & \multicolumn{5}{c|}{ImageNet} \\ \hline
Ensemble Team & 0123456789 & 0123 & 01238 & 123 & 1234 & 0123456789 & 12345 & 2345 & 1234 & 124 \\ \hline
Ensemble Acc (\%) & 96.33 & \textbf{97.15} & \textbf{96.87} & \textbf{96.81} & \textbf{96.63} & 79.82 & \textbf{80.77} & \textbf{80.70} & \textbf{80.29} & \textbf{79.84} \\ \hline
Team size & 10 & 4 & 5 & 3 & 4 & 10 & 5 & 4 & 4 & 3 \\ \hline
Acc Improv (\%) & 0 & \textbf{0.82} & \textbf{0.54} & \textbf{0.48} & \textbf{0.30} & 0 & \textbf{0.95} & \textbf{0.88} & \textbf{0.47} & \textbf{0.02} \\ \hline
Cost Reduction (\%) & 0 & \textbf{60} & \textbf{50} & \textbf{70} & \textbf{60} & 0 & \textbf{50} & \textbf{60} & \textbf{60} & \textbf{70} \\ \hline

\end{tabular}
} 
\end{table*}

\begin{table}[h!]
\centering
\caption{All Individual Member Models for Four Benchmark Datasets}
\label{table:ens-base-model-pools}
\scalebox{0.86}{
\small
\begin{tabular}{|c|c|c|c|c|c|c|c|c|}
\hline
\multirow{2}{*}{Dataset} & \multicolumn{2}{c|}{CIFAR-10}               & \multicolumn{2}{c|}{ImageNet}               & \multicolumn{2}{c|}{Cora}                  & \multicolumn{2}{c|}{MNIST}                  \\ \cline{2-9}
                         & \multicolumn{2}{c|}{10,000 testing samples} & \multicolumn{2}{c|}{50,000 testing samples} & \multicolumn{2}{c|}{1,000 testing samples} & \multicolumn{2}{c|}{10,000 testing samples} \\ \hline
ID                 & Name                  & Acc (\%)           & Name                    & Acc (\%)         & Name                  & Acc (\%)          & Name                      & Acc (\%)       \\ \hline
0                        & \textbf{DenseNet190}           & \textbf{96.68}              & AlexNet                 & 56.63            & GCN                   & 81.70             & KNN                       & 94.23          \\ \hline
1                        & DenseNet100           & 95.46              & DenseNet                & 77.15            & GAT                   & 82.80             & Logistic Regression       & 91.89          \\ \hline
2                        & ResNeXt               & 96.23              & EfficientNet-B0         & 75.80            & SGC                   & 81.70             & Linear SVM                & 92.48          \\ \hline
3                        & WRN                   & 96.21              & ResNeXt50               & 77.40            & ARMA                  & 82.10             & \textbf{RBF SVM}                   & \textbf{96.31}          \\ \hline
4                        & VGG19                 & 93.34              & Inception3              & 77.25            & APPNP                 & 82.20             & Random Forest             & 95.91          \\ \hline
5                        & ResNet20              & 91.73              & \textbf{ResNet152}               & \textbf{78.25}            & APPNP1                & 83.80             & GBDT                      & 92.89          \\ \hline
6                        & ResNet32              & 92.63              & ResNet18                & 69.64            & APPNP2                & 88.70             & Neural Network            & 96.18          \\ \hline
7                        & ResNet44              & 93.10              & SqueezeNet              & 58.00            & \textbf{SplineCNN}             & \textbf{88.90}             &                           &                \\ \hline
8                        & ResNet56              & 93.39              & VGG16                   & 71.63            & SplineCNN1            & 88.30             &                           &                \\ \hline
9                        & ResNet110             & 93.68              & VGG19-BN                & 74.22            & SplineCNN2            & 88.50             &                           &                \\ \hline
MIN                      & ResNet20              & 91.73              & AlexNet                 & 56.63            & GCN/SGC               & 81.70             & Logistic Regression       & 91.89          \\ \hline
AVG                      &                       & 94.25              &                         & 71.60            &                       & 84.87             &                           & 94.27          \\ \hline
MAX                      & DenseNet190           & 96.68              & ResNet152               & 78.25            & SplineCNN             & 88.90             & RBF SVM                   & 96.31  \\ \hline       
\end{tabular}
}
\end{table}

\newcommand{\tabfigure}[2]{\raisebox{-0.5\height}{\includegraphics[#1]{#2}}}
\begin{table*}[h!]
\caption{{\small Examples on ImageNet and Top-3 Classification Confidence}}
\label{table:ens-prune-example-imagenet}
\scalebox{0.87}{
\centering
\small
\begin{tabular}{|c|c|c|c|}
\hline
Image & 
\tabfigure{width=0.27\textwidth}{vis/46426-image}
&
\tabfigure{width=0.27\textwidth}{vis/47592-image}
&
\tabfigure{width=0.27\textwidth}{vis/39750-image}
\\ [9ex] \hline
\makecell{Ground \\Truth Label} & ice cream & lemon & ski \\ \hline
\makecell{Accuracy (\%)\\$F_1 F_2 F_4$: 79.84\\(DenseNet: 77.15,\\EfficientNet-B0: 75.80,\\Inception3: 77.25)} &
\tabfigure{width=0.275\textwidth}{vis/46426-124}
&
\tabfigure{width=0.275\textwidth}{vis/47592-124}
&
\tabfigure{width=0.275\textwidth}{vis/39750-124}
\\ \hline
\makecell{Ensemble \\Output} & ice cream & lemon & ski \\ \hline
\end{tabular}
} 
\centering
\end{table*}

We show 5 example deep ensemble teams in Table~\ref{table:example-ensembles-cifar10-imagenet} for two datasets, CIFAR-10 and ImageNet. For each dataset, we list the entire deep ensemble (\verb|0123456789|) of 10 member networks in addition to the 4 high quality deep ensembles that are identified by our HQ ensemble pruning approach.
The 10 member networks in each entire ensemble for CIFAR-10 and ImageNet are given in Table~\ref{table:ens-base-model-pools}. The 4 sub-ensembles recommended by our HQ have much smaller team sizes with only 3$\sim$5 individual member networks and achieve better ensemble accuracy than the given entire ensemble of 10 networks, significantly reducing the ensemble execution cost by 50\%$\sim$70\% for both CIFAR-10 and ImageNet.
Table~\ref{table:ens-prune-example-imagenet} further presents 3 image examples from ImageNet and the prediction results with Top-3 classification confidence from the member networks of the sub-ensemble team \verb|124|. This ensemble team achieves higher accuracy than each individual member network. For each image, one member model makes a prediction error. But the ensemble team can still give the correct predictions by repairing the wrong predictions by its member networks. This sub-ensemble team of 3 member networks also outperforms the entire ensemble of 10 models on ImageNet.
For a given deep ensemble of size $M=10$, we have a total of 1012 possible sub-ensembles to be considered in ensemble pruning. It is challenging to design and develop an efficient ensemble pruning approach to identify high quality sub-ensembles to improve both ensemble accuracy and time and space efficiency of ensemble execution.

\noindent {\bf Problems with Baseline Ensemble Pruning.\/}
In related work, we discussed several recent studies that utilize deep ensembles to strengthen the robustness of individual deep neural network against adversarial attacks~\cite{ensemble-mass,ensemble-tdsc,ensemble-bigdata,kaho-sigkdd2021}. 
Most of these existing methods leverage Cohen's Kappa (CK)~\cite{cohenskappa} for measuring ensemble diversity since early studies~\cite{diversityaccuracy,analysis-diversity-measures,ensemble-pruning-primer,ensemblepruninganalysis} in the literature have shown that both pairwise and non-pairwise diversity metrics share similar diversity evaluation results with regard to ensemble predictive performance, including the CK, BD, KW and GD metrics mentioned in related work.
However, we show that these existing diversity metrics may not precisely capture the inherent failure independence among member networks of a deep ensemble team, which may not produce the optimal performance in guiding ensemble pruning.

We first study the baseline diversity metrics for ensemble pruning and analyze their inherent problems. For a given diversity metric, such as BD or GD, the baseline ensemble pruning method calculates the ensemble diversity scores for each candidate ensemble in $EnsSet$ using a set of random samples drawn from the validation set as suggested by~\cite{diversityaccuracy}. 
A practical method for pruning a given large ensemble of $M$ member networks follows three steps: (1) we first compute the ensemble diversity score for each candidate sub-ensemble by using CK, BD, KW or GD; (2) we calculate the mean diversity score as the pruning threshold; and (3) we choose these sub-ensembles in $EnsSet$ with their ensemble diversity scores above this pruning threshold and add them into the set $GEnsSet$ of high quality ensembles. The rest sub-ensembles will be pruned out given their lower ensemble diversity than the pruning threshold.

\begin{figure}[h!]
\centering
    \subfloat[\small{\textbf{GD}, Pruning 427 out of 1012}]{
    \centering
    \includegraphics[width=0.45\textwidth]{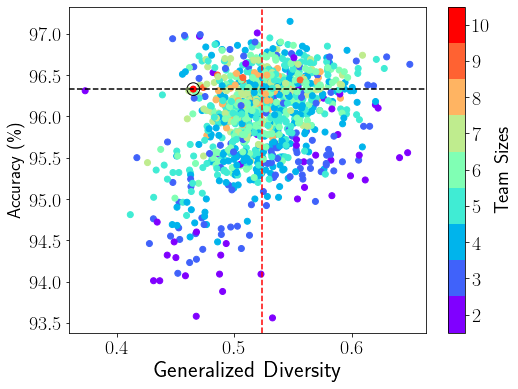}
    \label{fig:cifar10-q-gd-mean}
    }
    \subfloat[\small{\textbf{F-GD}, Pruning 517 out of 1012}]{
    \centering
    \includegraphics[width=0.45\textwidth]{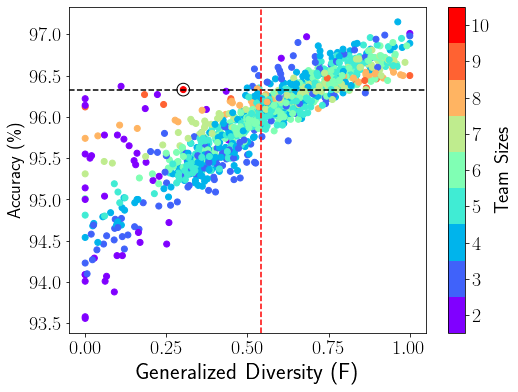}
    \label{fig:cifar10-f-gd-mean}
    }
    \caption{Pruning All Possible Deep Ensembles by The Mean Threshold on CIFAR-10: (a) Baseline Diversity (GD), (b) Focal Diversity (F-GD)}
    \label{fig:cifar10-mean}
\end{figure}

We compute the baseline GD scores and ensemble accuracy for all 1013 (1012 sub-ensembles $+$ 1 entire ensemble) deep ensembles on CIFAR-10 as shown in Figure~\ref{fig:cifar10-q-gd-mean}, where each dot denotes an ensemble team and each color marks the team size in the right color mapping bar. The baseline GD ensemble pruning approach utilizes the GD mean threshold of 0.524 (vertical red dashed line) to select sub-ensemble teams on the right of this diversity threshold line into $GEnsSet$. The ensemble accuracy 96.33\% of the entire ensemble (marked by the black circle) is represented by the horizontal black dashed line, which serves as the reference accuracy for evaluating ensemble pruning methods. Ideally, the ensemble pruning algorithms should identify high quality sub-ensembles with their ensemble accuracy above this reference accuracy.

For an entire ensemble of a large size $M$, we use four performance metrics to assess the ensemble pruning algorithms: (1) the accuracy range of the selected sub-ensembles in $GEnsSet$, (2) the precision, which measures the proportion of the selected sub-ensembles with equal or higher ensemble accuracy than the reference accuracy of the given entire ensemble, e.g., 96.33\% accuracy for CIFAR-10 and 79.82\% accuracy for ImageNet, among all selected sub-ensembles, (3) the recall, which measures the proportion of the selected sub-ensembles in $GEnsSet$ over all the sub-ensembles in $EnsSet$ whose ensemble accuracy is equal to or higher than the reference accuracy of the entire deep ensemble, and (4) the cost reduction, which measures the reduction in the ensemble team size of the selected sub-ensembles relative to the entire ensemble size $M$. For example, we show a small sub-ensemble, \texttt{123}, in Table~\ref{table:example-ensembles-cifar10-imagenet} with the cost reduction of 70\% ($=(M-S)/M=(10-3)/10$).
We evaluate the baseline mean threshold based ensemble pruning algorithm using the baseline BD (2nd column) and GD (4th column) diversity metrics on the CIFAR-10 dataset in Table~\ref{table:cifar10-mean} and ImageNet dataset in Table~\ref{table:imagenet-mean}. The BD-based baseline pruning suffers from low precision (6.18\% for CIFAR-10 and 8.53\% for ImageNet) and recall (11.72\% for CIFAR-10 and 23.15\% for ImageNet). Even though the GD-based baseline pruning has higher precision (36.90\% for CIFAR-10 and 17.74\% for ImageNet) and recall (63.10\% for CIFAR-10 and 46.31\% for ImageNet), which is still below the acceptable values, such as precision of over 50\%.
Moreover, the accuracy lower bound for both BD and GD based baseline ensemble pruning, i.e., 93.56\% on CIFAR-10 and 61.39\% (BD) and 70.79\% (GD) on ImageNet, is much lower than the reference accuracy of the entire ensemble team, i.e., 96.33\% accuracy for CIFAR-10 and 79.82\% accuracy for ImageNet.
We find similar observations on other diversity metrics, including CK and KW. This motivates us to investigate how to properly measure ensemble diversity and optimize the ensemble pruning process.

\begin{table}[h!]
\centering
\caption{Baseline Ensemble Pruning by Mean Threshold (CIFAR-10)}
\label{table:cifar10-mean}
\scalebox{0.9}{
\small
\begin{tabular}{|c|c|c|c|c|}
\hline
Methods & BD$>$0.339 (baseline) & F-BD$>$0.602 (optimized) & GD$>$0.524 (baseline) & F-GD$>$0.543 (optimized) \\ \hline
\begin{tabular}[c]{@{}c@{}}Acc Range (\%) \\ of $GEnsSet$  \end{tabular} & 93.56$\sim$96.72 & \textbf{95.71}$\sim$97.15 & 93.56$\sim$97.15 & \textbf{95.71}$\sim$97.15 \\ \hline
Precision (\%) & 6.18 & \textbf{63.53} & 36.90 & \textbf{53.90} \\ \hline
Recall (\%) & 11.72 & \textbf{95.51} & 63.10 & \textbf{97.59} \\ \hline
Cost Reduction & 10\%$\sim$80\% & 10\%$\sim$80\% & 10\%$\sim$80\% & 10\%$\sim$80\% \\ \hline
\end{tabular}
} 
\end{table}

\begin{table}[h!]
\centering
\caption{Baseline Ensemble Pruning by Mean Threshold (ImageNet)}
\label{table:imagenet-mean}
\scalebox{0.9}{
\small
\begin{tabular}{|c|c|c|c|c|}
\hline
Methods & BD$>$0.314 (baseline) & F-BD$>$0.632 (optimized) & GD$>$0.335 (baseline) & F-GD$>$0.594 (optimized) \\ \hline
\begin{tabular}[c]{@{}c@{}}Acc Range (\%) \\ of $GEnsSet$  \end{tabular} & 61.39$\sim$80.54 & \textbf{76.77}$\sim$80.77 & 70.79$\sim$80.60 & \textbf{76.41}$\sim$80.77 \\ \hline
Precision (\%)             & 8.53             & \textbf{41.93}            & 17.74            & \textbf{36.75}            \\ \hline
Recall (\%)                & 23.15            & \textbf{98.52}            & 46.31            & \textbf{99.01}            \\ \hline
Cost Reduction             & 10\%$\sim$80\%   & 10\%$\sim$80\%   & 10\%$\sim$80\%   & 10\%$\sim$80\%   \\ \hline
\end{tabular}
} 
\end{table}

\noindent {\bf Focal Diversity based Hierarchical Pruning.\/}
We improve the ensemble diversity measurement and ensemble pruning process with three novel techniques in this paper.
{\it First,} we use focal diversity metrics to optimize the ensemble diversity measurements, which utilizes the concept of focal models to sample negative samples and precisely capture the failure independence of member networks of a deep ensemble team.
{\it Second,} we introduce a novel hierarchical pruning approach, which leverages focal diversity metrics and iteratively prunes out subsets of redundant member networks with low diversity from the entire ensemble.
{\it Third,} we combine multiple focal diversity metrics through focal diversity consensus voting to further enhance the hierarchical ensemble pruning performance.

\section{Focal Diversity Based Hierarchical Pruning}\label{section:focal-diversity-hierarchical pruning}
\subsection{Focal Diversity Concept and Algorithm}~\label{section:focal-diversity-algorithm}
Our focal diversity metrics are designed to provide more accurate measurements of failure independence of the member networks in a deep ensemble team, including both pairwise (F-CK and F-BD) and non-pairwise (F-GD and F-KW) diversity metrics. 
Unlike traditional ensemble diversity evaluation approaches, which randomly select samples from the validation set to evaluate ensemble diversity, our focal diversity measurements draw random negative samples from $NegSampSet(F_f)$ based on a specific focal model $F_f$ on which the focal model $F_f$ makes prediction errors and then calculate the focal negative correlation score. For a given deep ensemble of $S$ member networks, each of the $S$ member networks will serve as the focal model ($F_f$) to draw negative samples. Therefore, we have a total of $S$ focal negative correlation scores, where each corresponds to a member network serving as the focal model. Then we obtain the focal diversity score for this deep ensemble of $S$ member networks through the (weighted) average of $S$ focal negative correlation scores.
The focal model concept is motivated by ensemble defense against adversarial attacks~\cite{ensemble-bigdata,ensemble-tdsc}, where the attack victim model is protected in a defense ensemble team. The focal model can be viewed as the victim model to evaluate the failure independence of other member models to this focal model in an ensemble team, i.e., the focal negative correlation score. Thus, the ensemble diversity score for this ensemble of size $S$ can be computed as the (weighted) average of $S$ focal negative correlation scores by taking each member model as a focal model to mitigate the bias in diversity evaluation using only one focal model.
Our preliminary results have demonstrated promising performance of the focal diversity metrics in capturing the failure independence of the member networks of a deep ensemble team~\cite{EnsembleBenchCogMI,EnsembleBenchCVPR,EnsembleBenchICDM}.

\begin{algorithm}[!h]
\caption{Focal Diversity Metric Calculation}
\label{alg:hq-diversity-metric}
\small
\begin{algorithmic}[1]
    \Procedure{getFQ}{$NegSampSet, Q, EnsSet$}
    \State \textbf{Input}: $NegSampSet$: negative sample sets for each focal model $F_f$; $Q$: the focal diversity metric, such as F-CK, F-BD, F-KW and F-GD; $EnsSet$: the set of candidate sub-ensemble teams to be considered;
    \State \textbf{Output}: $FQ$: focal diversity scores
    \State Initialize $D(Q)=\{\}$, $\overline{D}(Q)=\{\}$
    \State Initialize $FQ=\{\}$ \Comment{A map of diversity scores and teams}
    \For{$S=2$ to $M-1$} \label{alg:hq-s-start} 
        \For{$f = 0$ to $M-1$} \label{alg:hq-s-focal-start}
            \State Obtain $EnsSet(F_{f}, S)$ with candidate sub-ensembles of size $S$ and containing the focal model $F_{f}$.
            \State Initialize $D(Q, S, F_{f}) = \{\}$
            \For{$i=1$ to $|EnsSet(F_{f}, S)|$} \label{alg:hg-q-diversity-start}
                \LineComment{calculate the focal negative correlation score for $T_i \in EnsSet(F_{f}, S)$ using the definition of $Q$}
                \State $q_i = FocalNegativeCorrelation(Q, T_i, NegSampSet(F_{f}))$
                \State $D(Q, S, F_{f})$.append($q_i$)
            \EndFor \label{alg:hg-q-diversity-end}
            \For{$i=1$ to $|EnsSet(F_{f}, S)|$} \label{alg:hg-scale-start}
                \LineComment{scale focal negative correlation scores for sub-ensembles of the same size $S$}
                \State $\overline{D}(Q, S, F_{f}, T_i) = \frac{q_i - \min(D(Q, S, F_{f}))}{\max(D(Q, S, F_{f})) - \min(D(Q, S, F_{f}))}$ 
            \EndFor \label{alg:hg-scale-end}    
        \EndFor \label{alg:hq-s-focal-end}
        \State Obtain $EnsSet(S)$ with candidate sub-ensemble teams of size $S$
        \For{$i=1$ to $|EnsSet(S)|$} \label{alg:hq-avg-start}
            \State Initialize $tmpD=\{\}$
            \For{$j=1$ to $|T_i|$}
                \State $tmpD$.append($\overline{D}(Q, S, F_{f} = T_i[j], T_i)$)
            \EndFor
            \LineComment{Obtain the member model accuracy ranks as the member model weights.}
            \State $w = MemberModelAccuracyRank(T_i)$
            \State $FQ(T_i) = WeightedAverage(w, tmpD)$
        \EndFor \label{alg:hq-avg-end}
    \EndFor \label{alg:hq-s-end}
    \State \Return $FQ$
    \EndProcedure
\end{algorithmic}
\end{algorithm}

In general, our focal diversity metrics compare the ensemble diversity scores among the ensembles of the same size $S$, randomly draw negative samples from each focal model ($F_{f}$, i.e., $NegSampSet(F_f)$) and calculate the focal diversity score as the average of $S$ focal negative correlation scores by taking each member model as the focal model.
Algorithm~\ref{alg:hq-diversity-metric} gives a skeleton of calculating the focal diversity scores for all the candidate sub-ensembles in $EnsSet$.
For each team size $S$ (Line~\ref{alg:hq-s-start}$\sim$\ref{alg:hq-s-end}), we follow two general steps to calculate the focal diversity scores for each ensemble.
{\it First,} for each member model $F_{f}$, let $EnsSet(F_{f}, S)$ denote all candidate ensembles of size $S$, each containing the focal model $F_{f}$. We first compute the focal negative correlation score for each ensemble in $EnsSet(F_{f}, S)$ with the negative samples randomly drawn from the focal model $F_{f}$ ($NegSampSet(F_f)$) and store them in $D(Q, S, F_{f})$ (Line~\ref{alg:hg-q-diversity-start}$\sim$\ref{alg:hg-q-diversity-end}). Then, in order to make them comparable across different member models ($F_{f}$), we scale $D(Q, S, F_{f})$ into $[0, 1]$ and store them in $\overline{D}(Q, S, F_{f}, T_i)$ for each ensemble team $T_i$ (Line~\ref{alg:hg-scale-start}$\sim$\ref{alg:hg-scale-end}).
{\it Second,} for each candidate sub-ensemble ($T_i$) of size $S$, we perform a weighted average of the scaled focal negative correlation scores $\overline{D}(Q, S, F_{f} = T_i[j], T_i)$ associated with each of its focal (member) model $F_{f}=T_i[j]$ to obtain the focal diversity score. The weight is calculated with the corresponding rank of the accuracy of the member model ($T_i[j]$) in the ensemble team ($T_i$), i.e., the member models with higher accuracy will have higher weights (Line~\ref{alg:hq-avg-start}$\sim$\ref{alg:hq-avg-end}). In addition, we subtract the CK value from 1 when using the CK formula~\cite{cohenskappa} to present the consistent view that high diversity values correspond to high ensemble diversity.

We show a visual comparison between our focal diversity metric F-GD in Figure~\ref{fig:cifar10-f-gd-mean} and the baseline GD metric in Figure~\ref{fig:cifar10-q-gd-mean} for CIFAR-10. Here the baseline mean threshold based ensemble pruning method is used for both F-GD and GD diversity metrics. Compared to the GD based ensemble pruning, the F-GD powered pruning obtains better ensemble pruning results and identifies a larger portion of sub-ensembles (on the right side of the vertical red dashed line) with ensemble accuracy above the reference accuracy of 96.33\% (horizontal black dashed line).

When the focal diversity metrics are used in our hierarchical pruning algorithm with a desired team size $S_d$, we only calculate the diversity scores for the ensembles of size $2\sim S_d$ by replacing $M-1$ with $S_d$ in Line~\ref{alg:hq-s-start} of Algorithm~\ref{alg:hq-diversity-metric} to reduce the computation cost. Compared to the baseline diversity metrics, our focal diversity score computation takes about 1.2$\sim$2.4$\times$ the time of the baseline diversity score computation for all candidate sub-ensembles in $EnsSet$. The computation of all focal diversity metrics can be finished in several seconds on MNIST, CIFAR-10 and Cora and in several minutes on ImageNet (see Table~\ref{table:pruning-execution-time}). Given that our focal diversity metrics significantly outperform the baseline diversity metrics in capturing the failure independence of a group of member models, it is beneficial and worthwhile to use our focal diversity metrics to measure the ensemble diversity and identify high quality sub-ensembles. In addition, our hierarchical pruning can effectively prune out about 80\% of the ensemble teams in $EnsSet$, which also compensates for the increased computation time of the focal diversity metrics, making our entire solution very efficient.

We show the baseline mean threshold based ensemble pruning results using our focal diversity F-BD and F-GD in the 3rd and 5th columns of Table~\ref{table:cifar10-mean} for CIFAR-10 and Table~\ref{table:imagenet-mean} for ImageNet. Both F-BD and F-GD (optimized) substantially outperform the baseline BD and GD based ensemble pruning, as measured by the accuracy range of the selected sub-ensembles in $GEnsSet$, precision and recall. For the other two diversity metrics, CK and KW, we observed similar performance improvements of using our focal diversity metrics, F-CK and F-KW, in the baseline ensemble pruning. Even though our focal diversity metrics can significantly improve the baseline mean threshold based ensemble pruning method, achieving very high recall of over 95\%, we can still enhance the 63.53\% precision of F-BD and 53.90\% precision of F-GD for CIFAR-10 and the 41.93\% precision of F-BD and 36.75\% precision of F-GD for ImageNet to more accurately identify high-quality sub-ensembles. In addition, the baseline mean threshold based ensemble pruning examines every sub-ensemble in $EnsSet$ with high computational cost. These observations have motivated us to explore new methods to enhance ensemble pruning efficiency and accuracy.

\begin{algorithm}[!h]
\caption{Hierarchical Pruning}
\label{alg:alpha-filter}
\small
\begin{algorithmic}[1]
    \Procedure{hq-pruning}{$Q, S_d, \beta, EnsSet$}
    \State \textbf{Input}: $Q$: the focal diversity metric; $S_d$: the desired ensemble size; $\beta$: the percentage of the number of ensemble teams to be pruned out in each iteration; $EnsSet$: the set of sub-ensemble teams to be considered;
    \State \textbf{Output}: $GEnsSet(Q, S_d)$: the set of good ensemble teams of size $S_d$ identified by the focal diversity metric $Q$.
    \State Initialize $pruneSet = \{\}$ \Comment{Subsets of member models to prune out.}
    \For{$S=2$ to $S_d$}
        \State Initialize $GEnsSet(Q, S)=\{\}$, $D=\{\}$
        \State Construct $EnsSet(S)$ of size $S$ ensembles \label{alg:alpha-filter-prune-start}
        \For{$i=1$ to $|EnsSet(S)|$}
            \If{$T_i$ contains any subset in $pruneSet$}
                \State continue \Comment{Prune out this sub-ensemble $T_i$ and avoid expanding along this branch.}
            \Else
                \State $q_i=FQ(T_i)$ \Comment{$FQ(T_i)$ is obtained through Algorithm~\ref{alg:hq-diversity-metric} by using the focal diversity $Q$.}
                \State $D$.append($q_i$)
                \State $GEnsSet(Q, S)$.add($T_i$)
            \EndIf
        \EndFor \label{alg:alpha-filter-prune-end}
        \State $n = \beta \times |GEnsSet(Q, S)|$ \label{alg:alpha-filter-beta-start}
        \State Sort $T_i \in GEnsSet(Q, S)$ by $q_i \in D$
        \State Remove $n$ sub-ensembles of the lowest ensemble diversity from $GEnsSet(Q, S)$ and add them into $pruneSet$ \label{alg:alpha-filter-beta-end}
    \EndFor
    \State \Return $GEnsSet(Q, S_d)$
    \EndProcedure
\end{algorithmic}
\end{algorithm}

\subsection{Hierarchical Pruning Overview and Algorithm}
Our focal diversity metrics reveal some anti-monotonicity property: a superset of a low diversity ensemble has also low diversity. Specifically, a low focal diversity score for an ensemble team, e.g., $F_0 F_2$, often implies insufficient ensemble diversity, that is a high correlation of its member models in making similar prediction errors, making the member models highly redundant for a large superset ensemble team, such as $F_0F_1F_2$. Hence, those large ensembles that contain a small sub-ensemble with low diversity (e.g., $F_0 F_2$), such as $F_0 F_1 F_2$, $F_0 F_2 F_3$, $F_0 F_1 F_2 F_3$, and $F_0 F_1 F_2 F_4$ tend to have lower ensemble diversity than other ensembles with the same size and more diverse member models. Therefore, when we identify a sub-ensemble with low ensemble diversity, these large ensembles that are the supersets of this sub-ensemble can be preemptively pruned out, which motivates our hierarchical ensemble pruning approach.
The pseudo code is given in Algorithm~\ref{alg:alpha-filter}. 
Overall, our hierarchical pruning is an iterative process of composing and selecting deep ensembles for a desired ensemble team size $S_d$.
{\it First,} we start the hierarchical pruning process with the ensembles of size $S=2$, that is $|EnsSet(S=2)|=\binom{M}{2}=10(10-1)/2=45$ candidate ensemble teams.
{\it Second,} for a given focal diversity metric, we rank candidate ensembles of size $S$, such as $S=2$, by their focal diversity scores ($q_i$), prune out the bottom $\beta$ percentage of ensembles with low focal diversity scores, and add them into $pruneSet$ as the pruning targets (Line~\ref{alg:alpha-filter-beta-start}$\sim$\ref{alg:alpha-filter-beta-end}). Here, a dynamic $\beta$ can be configured to accommodate the concrete number of candidate ensembles and diversity score distribution. A conservative strategy is recommended to set a small $\beta$, e.g., by default $\beta=10\%$.
{\it Third,} we preemptively prune out all subsequent ensemble teams with a larger size $S+1$ that contain one or more pruned sub-ensembles in $pruneSet$ (Line~\ref{alg:alpha-filter-prune-start}$\sim$\ref{alg:alpha-filter-prune-end}). By iterating through these pruning steps, our hierarchical pruning can efficiently identify high quality sub-ensembles of the desired team size $S_d$ and add them into $GEnsSet(Q, S_d)$.

\begin{figure}[h!]
\centering
    \includegraphics[width=0.9\linewidth]{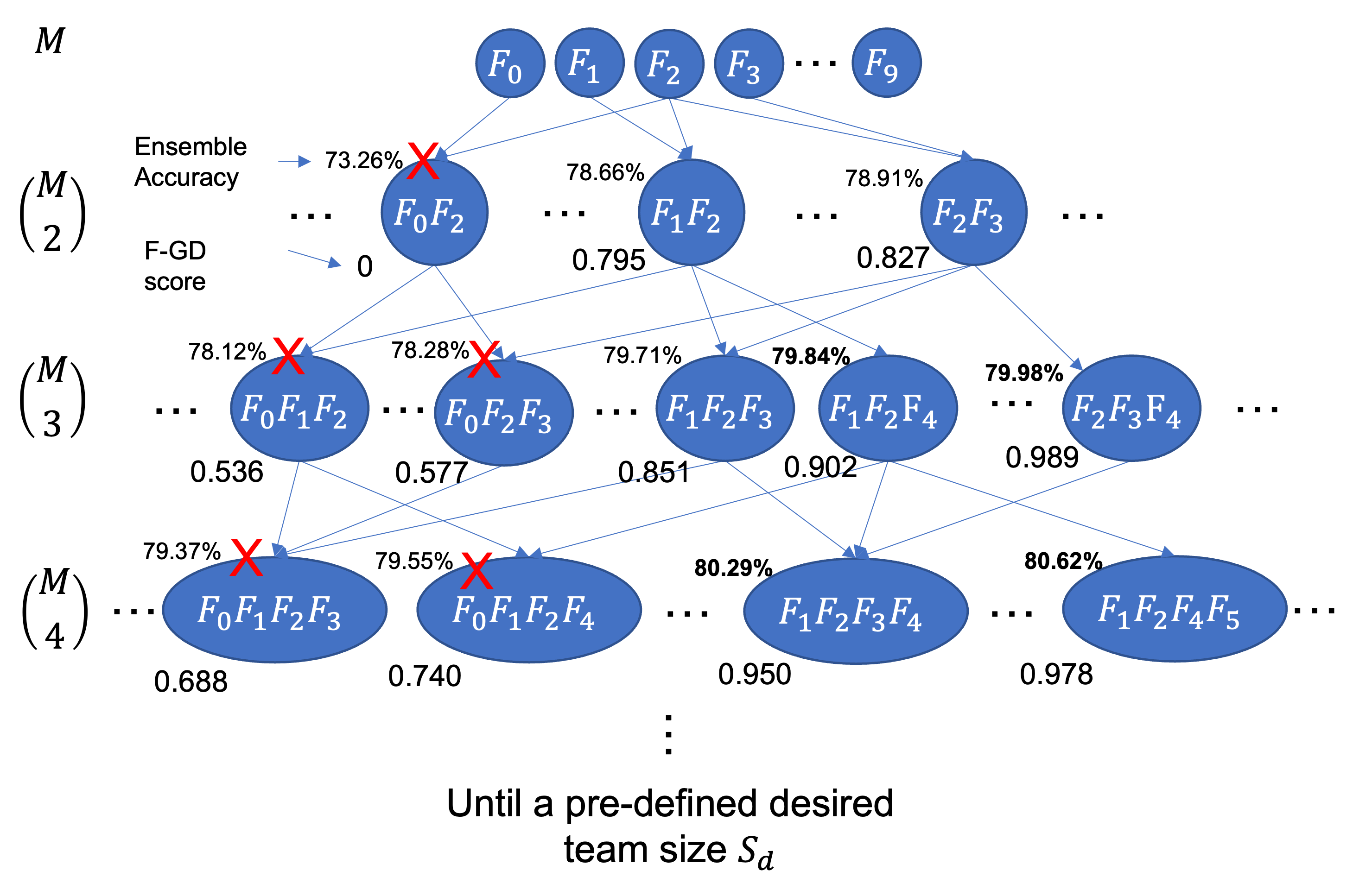}
    \caption{Hierarchical Pruning On ImageNet}
    \label{fig:hierarchical-pruning}
\end{figure}

We present a hierarchical ensemble pruning example on ImageNet in Figure~\ref{fig:hierarchical-pruning}. With $M=10$, we have all 10 member networks on the top, followed by all sub-ensembles of size $S=2$. For each tier, we add one additional network to the candidate ensemble teams of size $S~(2\le S < S_d)$ and place these extended ensembles of size $S+1$ in the next tier until $S=S_d$, where the ensembles of the desired team size $S_d$ will be placed in the last tier. Meanwhile, for each pruned ensemble, such as $F_0F_2$, our hierarchical pruning algorithm will preemptively cut off these branches of candidate ensemble teams that contain this pruned ensemble, i.e., all ensembles that are supersets of $F_0F_2$, marked by the red cross in Figure~\ref{fig:hierarchical-pruning}. This way of building ensemble teams allows us to compose high quality sub-ensembles and strategically remove the ensembles with low diversity. As the example in Figure~\ref{fig:hierarchical-pruning} shows that our hierarchical pruning indeed can effectively prune out these ensemble teams with low diversity and low accuracy, such as $F_0F_1F_2$, $F_0F_2F_3$, $F_0F_1F_2F_3$ and $F_0F_1F_2F_4$, and avoid exploring those unpromising branches.

We first apply our F-BD and F-GD hierarchical ensemble pruning algorithms on CIFAR-10 to prune the entire deep ensemble of 10 member networks by setting $\beta=10\%$ and $S_d=5$. Figure~\ref{fig:cifar10-hq-a-pruning-bd-gd} visualizes the ensemble pruning results, where the black and red dots denote the pruned and selected ensemble teams respectively. Three interesting observations should be highlighted.
{\it First,} for our focal diversity metrics, including F-BD and F-GD, the ensemble teams with high focal diversity scores tend to have high ensemble accuracy, especially when we compare the ensemble teams of the same size $S$, e.g., $S$=3, 4, 5, which is an important property to guide ensemble pruning.
{\it Second,} our hierarchical ensemble pruning approach can effectively identify these ensemble teams with insufficient ensemble diversity by their low focal diversity scores. The effectiveness of our hierarchical pruning can be attributed to two factors. On the one hand, our focal diversity powered hierarchical pruning encourages more fair comparison of focal diversity scores among these ensembles of the same size $S$. 
Comparing to Figure~\ref{fig:cifar10-f-gd-mean} showing all possible ensembles of mixed team sizes, Figure~\ref{fig:cifar10-hq-a-pruning-bd-gd} shows much clearer correlation between the focal diversity score and ensemble accuracy for the ensembles of the same size $S=3, ~4,~5$. On the other hand, our focal diversity metrics can more precisely capture the failure independence of member networks of a deep ensemble with some anti-monotonicity property, ensuring that the focal diversity based ensemble pruning is highly accurate and efficient.
{\it Third,} the time and space complexity of the hierarchical pruning algorithm and pruned ensemble execution cost can be ultimately bounded by the desired team size $S_d$. For example, we recommend setting the desired ensemble team size $S_d$ up to $50\% \times M$, which allows our hierarchical pruning to effectively find these sub-ensembles that offer on par or improved ensemble accuracy over the entire ensemble of $M$ member networks and have a substantially smaller team size, such as one third or one half of $M$ with a significant cost reduction in ensemble execution. 
{\it Finally,} for Figure~\ref{fig:cifar10-hq-a-pruning-bd-t5} and~\ref{fig:cifar10-hq-a-pruning-gd-t5}, all the selected sub-ensembles (red dots) with the desired team size $S_d=5$ provide higher ensemble accuracy than the reference accuracy of 96.33\% for CIFAR-10, which leads to 100\% precision of our hierarchical pruning with F-BD and F-GD. We observe similar ensemble pruning results using the other two focal diversity metrics, F-CK and F-KW.

\begin{figure*}[h!]
\centering
    \subfloat[\small{S=3, (Pruning 35 Out of 120)}]{
    \centering
    \includegraphics[width=0.32\textwidth]{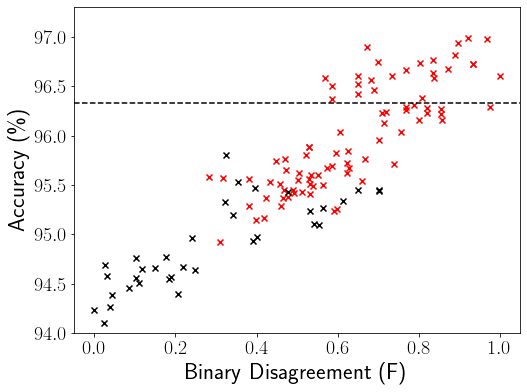}
    \label{fig:cifar10-hq-a-pruning-bd-t3}
    }
    \subfloat[\small{S=4, (Pruning 124 Out of 210)}]{
    \centering
    \includegraphics[width=0.32\textwidth]{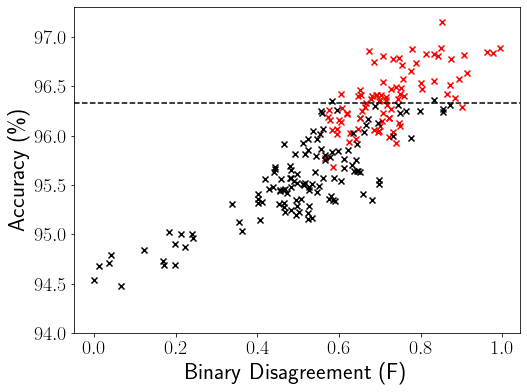}
    \label{fig:cifar10-hq-a-pruning-bd-t4}
    }
    \subfloat[\small{S=5, (Pruning 216 Out of 252)}]{
    \centering
    \includegraphics[width=0.32\textwidth]{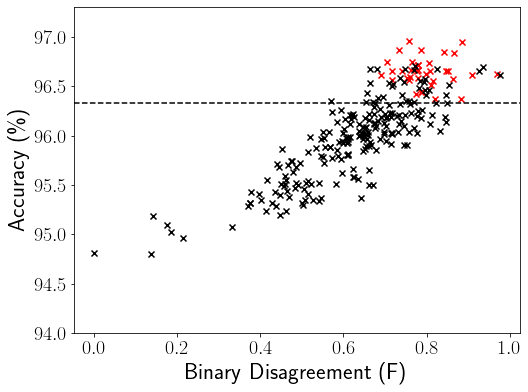}
    \label{fig:cifar10-hq-a-pruning-bd-t5}
    }\\
    \subfloat[\small{S=3, (Pruning 35 Out of 120)}]{
    \centering
    \includegraphics[width=0.32\textwidth]{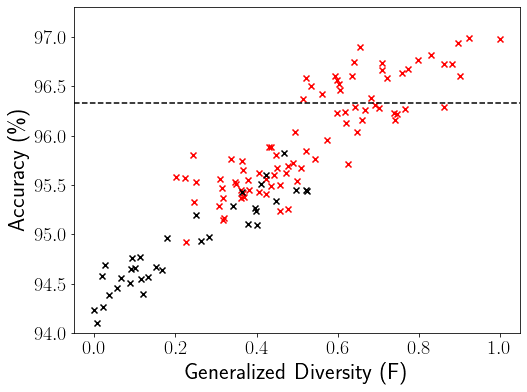}
    \label{fig:cifar10-hq-a-pruning-gd-t3}
    }
    \subfloat[\small{S=4, (Pruning 124 Out of 210)}]{
    \centering
    \includegraphics[width=0.32\textwidth]{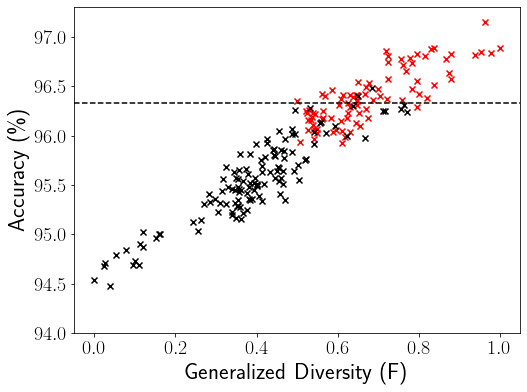}
    \label{fig:cifar10-hq-a-pruning-gd-t4}
    }
    \subfloat[\small{S=5, (Pruning 215 Out of 252)}]{
    \centering
    \includegraphics[width=0.32\textwidth]{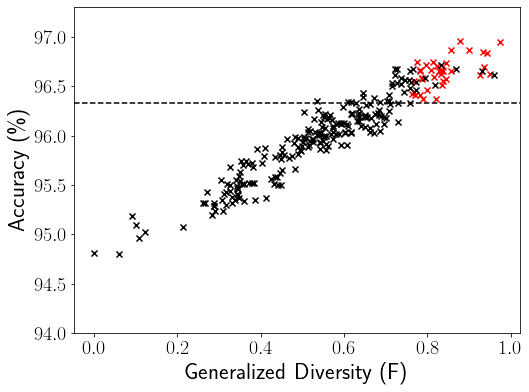}
    \label{fig:cifar10-hq-a-pruning-gd-t5}
    }
    \caption{Deep Ensembles of Size $S=3,4,5$ on CIFAR-10: top three figures for {\bf F-BD} and bottom three figures for {\bf F-GD} ($\beta=10\%, S_d=5$)}
    \label{fig:cifar10-hq-a-pruning-bd-gd}
\end{figure*}

\noindent \textbf{Focal Diversity Pruning by Focal Diversity Consensus Voting.\/}
Using different focal diversity metrics in the hierarchical ensemble pruning may produce different sets of selected ensembles ($GEnsSet$). It is observed that these ensemble teams that are chosen by the majority of our focal diversity metrics in ensemble pruning (e.g., F-BD pruning, F-KW pruning, F-GD pruning, etc.) tend to consistently outperform the entire deep ensemble in terms of the ensemble accuracy and execution cost. Therefore, we introduce the third step in our hierarchical ensemble pruning, which consolidates these ensembles that are selected by different focal diversity metrics into a majority voting based final selection. Concretely, an ensemble team will only be added to this final selection if it is chosen by the majority of focal diversity metrics, i.e., at least three focal diversity metrics, via hierarchical pruning. This third step further refines the focal diversity powered hierarchical pruning results and consistently delivers enhanced precision and robustness in ensemble pruning.

\section{Formal Analysis}
Deep neural network ensembles use multiple deep neural networks to form a team to collaborate and combine the predictions of individual member networks to make the final prediction. It is a commonly held view that an ensemble team consisting of diverse member networks has high prediction performance~\cite{ensemble-pruning-primer,ensemblepruninganalysis,effectivepruning}. Our focal diversity based ensemble pruning methods can efficiently identify small pruned deep ensembles with highly diverse member networks. These pruned deep ensembles can achieve the same or even improved prediction performance with significantly reduced space and time costs compared to the entire deep ensemble. In this section, we present a formal analysis to show the desired properties and features of these pruned deep ensembles to further demonstrate the effectiveness of our proposed methods.

\subsection{Diversity by Uncorrelated Errors}\label{section:uncorrelated-error}
For an ensemble team of size $S$, following~\cite{errorcorrelation,EnsembleBenchCVPR}, we can derive the added error for its ensemble prediction $E_{add}^{avg}$ using model averaging ($avg$) as the ensemble consensus method as Formula~(\ref{formula:uncorrelated-error-ensemble-added-error}) shows.
\begin{equation}
\small
E_{add}^{avg} = E_{add} (\frac{1 + (S-1) \delta)}{S})
\label{formula:uncorrelated-error-ensemble-added-error}
\end{equation}
where $E_{add}$ is the added error of a single network and $\delta$ is the expected average correlation of all member networks in the ensemble. Therefore, the ideal scenario is when all member networks in an ensemble team of size $S$ are diverse. They can learn and predict with uncorrelated errors (failure independence), i.e., $\delta \le 0$. Then a simple model averaging method can significantly reduce the overall prediction error by at least $S$ times. Meanwhile, the worst scenario happens when errors of individual networks are highly correlated with $\delta=1$. For example, when all $S$ member networks are perfect duplicates, the error of the ensemble is identical to the initial error without any improvement. In practice, given that it is challenging to directly measure the correlation $\delta$, many ensemble diversity metrics are proposed to quantify the correlation among member networks of an ensemble team. Our focal diversity metrics can significantly improve ensemble diversity measurement, and they are closely correlated to the ensemble prediction performance, which can be directly leveraged for identifying high-quality ensemble teams.

\subsection{Ensemble Robustness}\label{section:ensemble-robustness}
The deep neural network model is typically trained to minimize a cross-entropy loss and output a probability vector to approximate posteriori probabilities for the corresponding classes. Let $f_i(\textbf{x})$ denote the \textit{i}th element in the probability vector of a classifier $F$ for input $\textbf{x}$, which predicts the probability of class $i$. The classifier $F$ will output the predicted label $c$ with the highest probability for input $\textbf{x}$, that is $c=argmax_{1\le i \le C}f_i(\textbf{x})$.

The robustness of a classifier $F$ can be assessed based on its ability to maintain consistent prediction for a given input $\textbf{x}$ under input perturbation ($\mu$), such as the noise from adversarial samples. 
When the magnitude of the input perturbation $\mu$ remains within a certain bound, the prediction made by the classifier $F$ remains unaffected. We can leverage such a bound to compare the level of robustness of different classifiers, where a higher bound indicates a higher level of robustness. In Theorem~\ref{theorem:robustness-bound}, we introduce the concept of robustness bound ($R$) and formally show that this bound fulfills the aforementioned requirement.

\begin{theorem}[Robustness Bound ($R$)]\label{theorem:robustness-bound}
The robustness bound ($R$) for a classifier $F$ can be defined as the following Formula~(\ref{formula:lipschitz-inequality-6}),
\begin{equation}
\small
R = min_{j \ne c}\frac{f_c(\textbf{x}_0)-f_j(\textbf{x}_0)}{max_{\textbf{x}}||\nabla (f_c(\textbf{x}) - f_j(\textbf{x}))||_{q}}
\label{formula:lipschitz-inequality-6}
\end{equation}
where $\textbf{x} = \textbf{x}_0 + \mu$.
When the magnitude of the perturbation $\mu$ on the input $\textbf{x}_0$ is limited by $||\mu||_{p} \le R$ and $p,q$ satisfy $\frac{1}{p} + \frac{1}{q} = 1$ and $1\le p,q \le \infty$, the predicted labels for $\textbf{x}$ and $\textbf{x}_0$ by the classifier $F$ will be identical, indicating the input perturbation ($\mu$) will not change the prediction.
\end{theorem}

We first introduce Lemma~\ref{lemma:lipschitz-continuity} on Lipschitz continuity below and then show the formal proof of Theorem~\ref{theorem:robustness-bound}.

\begin{lemma}[Lipschitz Continuity]\label{lemma:lipschitz-continuity}
If $g(\textbf{x})$ is Lipschitz continuous, according to~\cite{lipschitz}, the following inequality~(\ref{formula:lipschitz-inequality}) holds:
\begin{equation}
\small
|g(\textbf{x}) - g(\textbf{y})| \le L_{q}||\textbf{x}-\textbf{y}||_{p}
\label{formula:lipschitz-inequality}
\end{equation}
where \textbf{x} and \textbf{y} are inputs, Lipschitz constant $L_{q} = max_{\textbf{x}}||\nabla g(\textbf{x})||_{q}$, $\frac{1}{p} + \frac{1}{q} = 1$, and $1\le p,q \le \infty$.
\end{lemma}

\begin{proof}[Proof of Theorem~\ref{theorem:robustness-bound}]

Without loss of generality, we assume $g(\textbf{x}) = f_c(\textbf{x}) - f_j(\textbf{x})$ is Lipschitz continuous with Lipschitz constant $L_{q}^{j}$ and $j \ne c$. Following Lemma~\ref{lemma:lipschitz-continuity}, let $\textbf{x} = \textbf{x}_0 + \mu$ and $\textbf{y}=\textbf{x}_0$, we derive Formula~(\ref{formula:lipschitz-inequality-1}) below,
\begin{equation}
\small
|g(\textbf{x}_0 + \mu) - g(\textbf{x}_0)| \le L_{q}^{j}||\mu||_{p}
\label{formula:lipschitz-inequality-1}
\end{equation}
where $L_{q}^{j} = max_{\textbf{x}}||\nabla g(\textbf{x})||_{q} =max_{\textbf{x}}||\nabla (f_c(\textbf{x}) - f_j(\textbf{x}))||_{q}$, $\frac{1}{p} + \frac{1}{q} = 1$, and $1\le p,q \le \infty$. 

Formula~(\ref{formula:lipschitz-inequality-1}) can be rearranged as Formula~(\ref{formula:lipschitz-inequality-2}),
\begin{equation}
\small
g(\textbf{x}_0) - L_{q}^{j}||\mu||_{p} \le g(\textbf{x}_0 + \mu) \le g(\textbf{x}_0) + L_{q}^{j}||\mu||_{p}
\label{formula:lipschitz-inequality-2}
\end{equation}
when $g(\textbf{x}_0 + \mu) < 0$, the predicted class label will change. However, $g(\textbf{x}_0 + \mu)$ is lower bounded by $g(\textbf{x}_0) - L_{q}^{j}||\mu||_{p} \le g(\textbf{x}_0 + \mu)$. If  $0 \le g(\textbf{x}_0) - L_{q}^{j}||\mu||_{p}$, we have $g(\textbf{x}_0 + \mu) \ge 0$ to ensure that the prediction will not change with the small perturbation $\mu$ on the input $\textbf{x}_0$. This leads to Formula~(\ref{formula:lipschitz-inequality-3}),
\begin{equation}
\small
g(\textbf{x}_0) - L_{q}^{j}||\mu||_{p} \ge 0 \Rightarrow ||\mu||_{p} \le \frac{g(\textbf{x}_0)}{L_{q}^{j}}
\label{formula:lipschitz-inequality-3}
\end{equation}
that is Formula~(\ref{formula:lipschitz-inequality-4}):
\begin{equation}
\small
||\mu||_{p} \le \frac{f_c(\textbf{x}_0)-f_j(\textbf{x}_0)}{L_{q}^{j}}
\label{formula:lipschitz-inequality-4}
\end{equation}
where $L_{q}^{j}=max_{\textbf{x}}||\nabla (f_c(\textbf{x}) - f_j(\textbf{x}))||_{q}$, $\frac{1}{p} + \frac{1}{q} = 1$, and $1\le p,q \le \infty$. 

In order to ensure the classification result will not change, that is $argmax_{1\le i \le C} f_i(\textbf{x}_0 + \mu) = c$, we use the minimum of the bound on $\mu$ over $j \ne c$ to obtain the inequality~(\ref{formula:lipschitz-inequality-5}) with $\frac{1}{p} + \frac{1}{q} = 1$ and $1\le p,q \le \infty$,
\begin{equation}
\small
||\mu||_{p} \le min_{j \ne c}\frac{f_c(\textbf{x}_0)-f_j(\textbf{x}_0)}{L_{q}^{j}} = min_{j \ne c}\frac{f_c(\textbf{x}_0)-f_j(\textbf{x}_0)}{max_{\textbf{x}}||\nabla (f_c(\textbf{x}) - f_j(\textbf{x}))||_{q}} = R
\label{formula:lipschitz-inequality-5}
\end{equation}
which indicates that as long as $||\mu||_{p}$ is small enough to fulfill the above bound, the classifier decision will never be changed, which marks the robustness of this classifier $F$. Hence, we formally prove Theorem~\ref{theorem:robustness-bound} on the robustness bound ($R$) for a classifier $F$. 
\end{proof}

For a deep neural network $F_k$, we have its robustness bound as Formula~(\ref{formula:robustness-bound-model-k}) shows.
\begin{equation}
\small
R^{k} = min_{j \ne c}\frac{f_c^k(\textbf{x}_0)-f_j^k(\textbf{x}_0)}{max_{\textbf{x}}||\nabla(f_c^k(\textbf{x})-f_j^k(\textbf{x}))||_{q}}
\label{formula:robustness-bound-model-k}
\end{equation}
Let $g_{j}^{k}(\textbf{x}) = f_{c}^{k}(\textbf{x}) - f_{j}^{k}(\textbf{x})$, we have the robustness bound as Formula~(\ref{formula:robustness-bound-model-k-2}) shows.
\begin{equation}
\small
R^{k} = min_{j \ne c}\frac{g_{j}^{k}(\textbf{x}_0)}{max_{\textbf{x}}||\nabla(g_{j}^{k}(\textbf{x}))||_{q}}
\label{formula:robustness-bound-model-k-2}
\end{equation}

Given $S$ networks, combining their predictions with model averaging ($avg$), we have the \textit{i}th element in the combined probability vector as $f_i^{avg}(\textbf{x}) = \frac{1}{S} \sum_{k=1}^{S}f_i^k(\textbf{x})$ corresponding to the robustness bound as the following Formula~(\ref{formula:robustness-bound-model-k-3}) shows.

\begin{equation}
\small
\begin{aligned}
R^{avg} &= min_{j \ne c}\frac{f_c^{avg}(\textbf{x}_0)-f_j^{avg}(\textbf{x}_0)}{max_{\textbf{x}}||\nabla(f_c^{avg}(\textbf{x})-f_j^{avg}(\textbf{x}))||_{q}} \\
&= min_{j \ne c}\frac{g_c^{avg}(\textbf{x}_0)}{max_{\textbf{x}}||\nabla(g_c^{avg}(\textbf{x}))||_{q}}
\end{aligned}
\label{formula:robustness-bound-model-k-3}
\end{equation}

Assume the minimum of the robustness bound can be achieved with the prediction result $c$ and $j$ for each model $F^{k}$ including the ensemble $F^{avg}$ as Formula~(\ref{formula:robustness-bound-model-k-4}) shows.
\begin{equation}
\small
\begin{aligned}
R^{k} &= \frac{g_{j}^{k}(\textbf{x}_0)}{max_{\textbf{x}}||\nabla(g_{j}^{k}(\textbf{x}))||_{q}} \\
R^{avg} &= \frac{g_{j}^{avg}(\textbf{x}_0)}{max_{\textbf{x}}||\nabla(g_{j}^{avg}(\textbf{x}))||_{q}}
\end{aligned}
\label{formula:robustness-bound-model-k-4}
\end{equation}
where $g_{j}^{k}(\textbf{x}) = f_{c}^{k}(\textbf{x}) - f_{j}^{k}(\textbf{x})$ and $g_j^{avg}(\textbf{x}) = \frac{1}{S} \sum_{k=1}^{S}g_j^k(\textbf{x})$. 

We can then derive Theorem~\ref{theorem:ensemble-robustness-bound-enhancement} that states $\exists~k, 1 \le k \le S, R^{k} \le R^{avg}$, where the equal sign corresponds to the case that all member networks are the perfect duplicates. This theorem further indicates that the ensembles of high diversity can improve the robustness of individual member networks.

\begin{theorem}[Ensemble Robustness Bound Enhancement]
\label{theorem:ensemble-robustness-bound-enhancement} 
Let $R^{avg}$ denote the robustness bound for an ensemble, which combines its member network predictions through model averaging ($avg$). We can always find one member network $F^{k}$ with its robustness bound $R^{k}$ satisfying $R^{k} \le R^{avg}$.
\end{theorem}

\begin{proof}[Proof of Theorem~\ref{theorem:ensemble-robustness-bound-enhancement}]
We prove the ensemble robustness bound enhancement by contradiction. First, we assume $\forall~k, 1 \le k \le S, R^{k} > R^{avg}$, that is as Formula~(\ref{formula:robustness-proof-1}) shows.
\begin{equation}
\small
g_{j}^{k}(\textbf{x}_0) (max_{\textbf{x}}||\nabla(g_{j}^{avg}(\textbf{x}))||_{q}) > g_{j}^{avg}(\textbf{x}_0) (max_{\textbf{x}}||\nabla(g_{j}^{k}(\textbf{x}))||_{q})
\label{formula:robustness-proof-1}
\end{equation}
For each $k\in \{1,...,S\}$, this inequality holds. To add them all, we have Formula~(\ref{formula:robustness-proof-2}).
\begin{equation}
\small
\sum_{k=1}^{S} g_{j}^{k}(\textbf{x}_0) (max_{\textbf{x}}||\nabla(g_{j}^{avg}(\textbf{x}))||_{q}) >
\sum_{k=1}^{S} g_{j}^{avg}(\textbf{x}_0) (max_{\textbf{x}}||\nabla(g_{j}^{k}(\textbf{x}))||_{q})
\label{formula:robustness-proof-2}
\end{equation}
That is Formula~(\ref{formula:robustness-proof-3}).
\begin{equation}
\small
(max_{\textbf{x}}||\nabla(g_{j}^{avg}(\textbf{x}))||_{q})\sum_{k=1}^{S} g_{j}^{k}(\textbf{x}_0) >
g_{j}^{avg}(\textbf{x}_0) \sum_{k=1}^{S} (max_{\textbf{x}}||\nabla(g_{j}^{k}(\textbf{x}))||_{q})
\label{formula:robustness-proof-3}
\end{equation}
Given $g_j^{avg}(\textbf{x}) = \frac{1}{S} \sum_{k=1}^{S}g_j^k(\textbf{x})$, we have Formula~(\ref{formula:robustness-proof-4}):
\begin{equation}
\small
(max_{\textbf{x}}||\nabla(\sum_{k=1}^{S} g_{j}^{k}(\textbf{x}))||_{q})\frac{1}{S}\sum_{k=1}^{S} g_{j}^{k}(\textbf{x}_0) >
\frac{1}{S}\sum_{k=1}^{S}g_{j}^{k}(\textbf{x}_0) \sum_{k=1}^{S} (max_{\textbf{x}}||\nabla(g_{j}^{k}(\textbf{x}))||_{q})
\label{formula:robustness-proof-4}
\end{equation}
Therefore, we have the following Formula~(\ref{formula:robustness-proof-5}).
\begin{equation}
\small
(max_{\textbf{x}}||\nabla(\sum_{k=1}^{S} g_{j}^{k}(\textbf{x}))||_{q}) > \sum_{k=1}^{S} (max_{\textbf{x}}||\nabla(g_{j}^{k}(\textbf{x}))||_{q})
\label{formula:robustness-proof-5}
\end{equation}
According to the triangle inequality, we have Formula~(\ref{formula:robustness-proof-6}).
\begin{equation}
\small
\begin{aligned}
max_{\textbf{x}}||\nabla(\sum_{k=1}^{S} g_{j}^{k}(\textbf{x}))||_{q} &\le
max_{\textbf{x}}(\sum_{k=1}^{S} ||\nabla(g_{j}^{k}(\textbf{x}))||_{q}) \\
&\le \sum_{k=1}^{S} (max_{\textbf{x}}||\nabla(g_{j}^{k}(\textbf{x}))||_{q})
\end{aligned}
\label{formula:robustness-proof-6}
\end{equation}
which contradicts with the derived inequality (Formula~(\ref{formula:robustness-proof-5})). Therefore, the previous assumption does not hold. We show that $\exists~k, 1 \le k \le S, R^{k} \le R^{avg}$, demonstrating that the robustness of a member network can be further improved with an ensemble team.
Furthermore, for a network $F^{k}$, if its robustness bound $R^{k}$ was not obtained with class $j$. We have $\exists i \ne j~(i,j \ne c)~\text{and}~ R^{k} = \frac{g_{i}^{k}(\textbf{x}_0)}{max_{\textbf{x}}||\nabla(g_{i}^{k}(\textbf{x}))||_{q}} \le \frac{g_{j}^{k}(\textbf{x}_0)}{max_{\textbf{x}}||\nabla(g_{j}^{k}(\textbf{x}))||_{q}}$, where Theorem~\ref{theorem:ensemble-robustness-bound-enhancement} still holds. 
\end{proof}

The above analysis formally certifies that an ensemble team of diverse member networks can further improve the robustness of individual networks.

\section{Experimental Evaluation}\label{section:experimental-evaluation}
We conducted a comprehensive experimental evaluation on four benchmark datasets, CIFAR-10, ImageNet, Cora and MNIST, for pruning the given entire ensemble of 10 individual member models for CIFAR-10, ImageNet and Cora and 7 member models for MNIST (see Table~\ref{table:ens-base-model-pools} for all member models). All the experiments were performed on an Intel i7-10700K server with the NVIDIA GeForce RTX 3090 (24GB) GPU on Ubuntu 20.04.

\begin{figure}[h!]
\centering
    \includegraphics[width=0.75\textwidth]{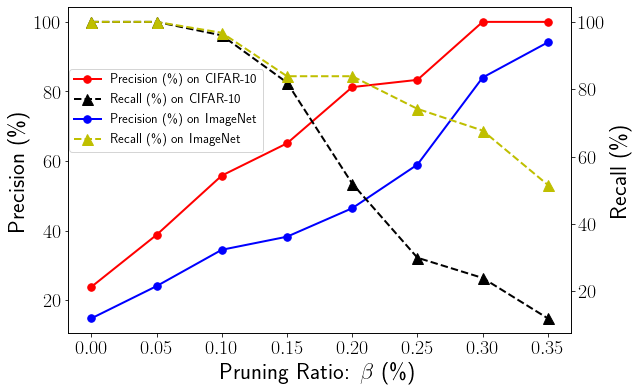}
    \caption{Impact of $\beta$ on Precision and Recall ($S_d=4$, F-GD)}
    \label{fig:precision-recall-beta}
\end{figure}

\subsection{Efficiency of Pruning with Varying $\beta$} \label{section:tuning-beta}
The hyperparameter $\beta$ in our hierarchical ensemble pruning determines the percentage of candidate ensembles to be pruned out during each iteration for a specific ensemble team size $S$. In general, a higher $\beta$ tends to remove more ensemble teams in our hierarchical pruning, potentially lowering the recall score. On the other hand, a high precision score is expected when the ensemble pruning algorithm can precisely find the high quality sub-ensembles with much smaller team sizes and equal or better ensemble accuracy of the reference accuracy of the entire ensemble team. We analyze the impacts of the hyperparameter $\beta$ in Figure~\ref{fig:precision-recall-beta}, where we vary $\beta$ from 5\% to 35\% for both CIFAR-10 and ImageNet using F-GD based hierarchical pruning and the desired team size $S_d=4$. As shown in Figure~\ref{fig:precision-recall-beta}, when the $\beta$ increases, the precision of ensemble pruning increases, and the recall decreases for both CIFAR-10 and ImageNet. Our focal diversity based hierarchical ensemble pruning is designed to achieve high precision and good recall, e.g., over 75\% precision and 50\% recall. Therefore, with $S_d=4$ on CIFAR-10, we set $\beta=20\%$ to obtain 81.25\% precision and 52\% recall of the F-GD based hierarchical pruning. We follow the same principles in determining the $\beta$ value for other experiments.

\begin{table}[h!]
\centering
\caption{Impact of Desired Team Size $S_d$ on Hierarchical Pruning (CIFAR-10)}
\label{table:cifar10-beta-impact}
\small
\scalebox{1.0}{
\begin{tabular}{|c|c|c|c|}
\hline
$S_d$             & 4     & 5     & 6     \\ \hline
$\beta$ (\%)      & 20    & 10    & 4     \\ \hline
Precision (\%) & 81.25 & \textbf{100}   & 97.56 \\ \hline
Recall (\%)    & 52    & 47.14 & 57.14 \\ \hline
\end{tabular}
} 
\end{table}

\subsection{Impact of Desired Team Size $S_d$}
The hyperparameter $S_d$ specifies the target ensemble team size after ensemble pruning. Table~\ref{table:cifar10-beta-impact} presents the impacts of varying $S_d$ on our focal diversity powered hierarchical pruning approach on CIFAR-10. For each $S_d$, we find the optimal $\beta$ following the principles introduced in Section~\ref{section:tuning-beta}. We highlight two interesting observations.
{\it First,} for different $S_d$, our hierarchical pruning can consistently deliver high precision of over 81.25\% with good recall, which demonstrates the effectiveness of our hierarchical pruning approach. In particular, $S_d$ = 5 produces the highest precision of 100\%.
{\it Second,} as the desired team size $S_d$ increases, there is a decrease in the optimal $\beta$ value, indicating that fewer ensembles will be pruned out for each iteration to achieve a good recall rate.
In practice, we recommend setting the desired ensemble team size $S_d$ up to half the size of the entire ensemble, which allows our hierarchical pruning to efficiently find these sub-ensembles that offer comparable or improved ensemble accuracy over the entire ensemble and substantially reduce the ensemble execution cost.

\subsection{Focal Diversity Pruning Methods}
We then evaluate our hierarchical ensemble pruning algorithm with four focal diversity metrics and the focal diversity consensus voting based refinement. The precision, recall and cost reduction of ensemble pruning are primarily used to measure the pruning efficiency on four benchmark datasets.

\begin{table}[h!]
\centering
\caption{Hierarchical Pruning with $S_d$=5 on CIFAR-10}
\label{table:cifar10-pruning-s5}
\small
\scalebox{1.0}{
\begin{tabular}{|c|c|c|c|c|c|}
\hline
Methods & F-CK & F-BD & F-KW & F-GD & MAJORITY-F \\ \hline
Precision (\%) & 85.71 & \textbf{100} & \textbf{100} & \textbf{100} & \textbf{100} \\ \hline
Recall (\%) & 17.14 & 51.43 & 51.43 & 52.86 & 47.14 \\ \hline
Cost Reduction & 50\% & 50\% & 50\% & 50\% & 50\% \\ \hline
\end{tabular}
} 
\end{table}

\begin{table}[h!]
\centering
\caption{Hierarchical Pruning with $S_d$=4 on CIFAR-10}
\label{table:cifar10-pruning-s4}
\small
\scalebox{1.0}{
\begin{tabular}{|c|c|c|c|c|c|}
\hline
Methods & F-CK & F-BD & F-KW & F-GD & MAJORITY-F \\ \hline
Precision (\%) & 26.09 & \textbf{81.25} & \textbf{81.25} & \textbf{81.25} & \textbf{81.25} \\ \hline
Recall (\%) & 12.00 & 52.00 & 52.00 & 52.00 & 52.00 \\ \hline
Cost Reduction & 60\% & 60\% & 60\% & 60\% & 60\% \\ \hline
\end{tabular}
} 
\end{table}

\noindent \textbf{CIFAR-10.\/} 
We compare the four focal diversity metrics using our hierarchical ensemble pruning and the focal diversity consensus voting based pruning (MAJORITY-F) in Table~\ref{table:cifar10-pruning-s5}, where we set $\beta=10\%, S_d=5$, and the F-BD and F-GD pruning results correspond to Figure~\ref{fig:cifar10-hq-a-pruning-bd-t5} and~\ref{fig:cifar10-hq-a-pruning-gd-t5}. Two interesting observations should be highlighted.
{\it First,} our hierarchical ensemble pruning approach achieves a very high precision of over 85\% in ensemble pruning with all four focal diversity metrics. Especially, three focal diversity metrics, F-BD, F-KW and F-GD, produces 100\% precision in identifying high quality sub-ensembles with much smaller team sizes and equal or better ensemble accuracy than the entire deep ensemble.
{\it Second,} with the focal diversity consensus voting based ensemble pruning, coined as MAJORITY-F, the 100\% precision is maintained.
MAJORITY-F further refines the ensemble pruning results by pruning out these ensemble teams that are not chosen by the majority of four focal diversity metrics, which slightly reduces the total number of selected ensembles (i.e., true positives given 100\% precision). Hence, the slightly lower recall score of MAJORITY-F is expected in comparison with F-BD, F-KW and F-GD. 
We find similar observations in Table~\ref{table:cifar10-pruning-s4} by setting $\beta=20\%, S_d=4$. Even though the F-CK based hierarchical pruning achieves lower performance than the other three focal diversity metrics, the focal diversity consensus voting method, MAJORITY-F, can still deliver consistent high precision of 81.25\% and recall of 52\%.

\begin{table}[h!]
\centering
\caption{Hierarchical Pruning with $S_d$=5 on ImageNet}
\label{table:imagenet-pruning-s5}
\small
\scalebox{1.0}{
\begin{tabular}{|c|c|c|c|c|c|}
\hline
Methods & F-CK & F-BD & F-KW & F-GD & MAJORITY-F \\ \hline
Precision (\%) & 0 & 75.61 & 75.61 & 74.42 & \textbf{78.95} \\ \hline
Recall (\%) & 0 & 64.58 & 64.58 & 66.67 & 62.50 \\ \hline
Cost Reduction & 50\% & 50\% & 50\% & 50\% & 50\% \\ \hline
\end{tabular}
} 
\end{table}

\noindent \textbf{ImageNet.\/} We then evaluate our hierarchical ensemble pruning approach on ImageNet by setting $\beta=10\%, S_d=5$. Table~\ref{table:imagenet-pruning-s5} shows the experimental results.
We highlight two interesting observations.
{\it First,} our hierarchical ensemble pruning approach performs very well with high precision of over 74.42\% by using F-BD, F-KW and F-GD metrics, which is effective and accurate in selecting high quality sub-ensembles with a much smaller team size ($S=5$ vs. $M=10$) and higher ensemble accuracy than the 79.82\% reference accuracy of the entire ensemble team on ImageNet. On the other hand, the F-CK based pruning failed to identify any satisfactory ensemble teams, which indicates the inherent limitation of the CK diversity, although it has been optimized by the focal diversity measures.
{\it Second,} using our focal diversity consensus voting based ensemble pruning, MAJORITY-F, we can effectively improve the precision by over 3.34\% from 74.42\%$\sim$75.61\% to 78.95\%, which further demonstrates the enhanced ensemble pruning performance by consolidating the majority of focal diversity metrics.

\begin{table}[h!]
\centering
\caption{Hierarchical Pruning with $S_d$=4 on Cora}
\label{table:cora-pruning-s4}
\small
\scalebox{1.0}{
\begin{tabular}{|c|c|c|c|c|c|}
\hline
Methods & F-CK & F-BD & F-KW & F-GD & MAJORITY-F \\ \hline
Precision (\%) & \textbf{91.67} & 86.30 & 86.30 & 84.93 & 87.14 \\ \hline
Recall (\%) & 64.71 & 61.76 & 61.76 & 60.78 & 59.80 \\ \hline
Cost Reduction & 60\% & 60\% & 60\% & 60\% & 60\% \\ \hline
\end{tabular}
} 
\end{table}

\noindent \textbf{Cora.\/} We also evaluate and compare our focal diversity based hierarchical pruning approach on Cora, a graph dataset. Table~\ref{table:cora-pruning-s4} shows the results with $\beta=10\%$ and $S_d=4$. All focal diversity based hierarchical pruning methods achieve over 84.93\% precision in identifying high quality sub-ensembles with the same or improved ensemble accuracy over 87.90\% (the reference accuracy of the entire ensemble for Cora). Interestingly, the F-CK pruning for Cora achieves the best ensemble pruning performance in terms of the highest precision of 91.67\% and highest recall of 64.71\%, indicating the diverse utilities of different focal diversity pruning methods for different datasets. Therefore, the focal diversity consensus voting based ensemble pruning is beneficial for delivering more stable and consistent ensemble pruning efficiency in terms of precision and robustness.

\begin{table}[h!]
\centering
\caption{Hierarchical Pruning with $S_d$=3 on MNIST}
\label{table:mnist-pruning-s3}
\small
\scalebox{1.0}{
\begin{tabular}{|c|c|c|c|c|c|}
\hline
Methods        & F-CK & F-BD & F-KW & F-GD & MAJORITY-F  \\ \hline
Precision (\%) & \textbf{100}  & 75   & 75   & 75   & 75   \\ \hline
Recall (\%)    & 60   & 60   & 60   & 60   & 60   \\ \hline
Cost Reduction         & 57\% & 57\% & 57\% & 57\% & 57\% \\ \hline
\end{tabular}
} 
\end{table}

\noindent \textbf{MNIST.\/} We then use popular machine learning models, such as KNN, SVM and Logistic Regression in Table~\ref{table:ens-base-model-pools} to evaluate the effectiveness of our focal diversity based hierarchical pruning methods on MNIST. 
Table~\ref{table:mnist-pruning-s3} shows the results with $\beta=50\%$ and $S_d=3$. We highlight two interesting observations.
{\it First,} all focal diversity based pruning methods achieved over 75\% precision in identifying sub-ensembles with the same or improved ensemble accuracy over 96.36\%, which is the reference accuracy of the entire ensemble for MNIST. It demonstrates that our focal diversity based hierarchical pruning methods are also effective on representative machine learning models.
{\it Second,} among the Top-3 selected ensemble teams by F-GD or F-CK, the ensemble team \verb|346| (RBF SVM, Random Forest and Neural Network), achieved the highest accuracy of 97.04\%, significantly outperforming the entire ensemble of 7 machine learning models with 96.36\% accuracy and reducing the ensemble execution cost by over 57\%. Moreover, given the Random Forest is already an ensemble of decision trees, it shows that integrating ensemble models, such as the Random Forest, in an ensemble team can further improve the overall predictive performance.

\subsection{Focal vs. Baseline Diversity in Hierarchical Pruning}
We next compare our focal diversity metrics with baseline diversity metrics by using the same hierarchical pruning algorithm. Figure~\ref{fig:cifar10-fq-q-pruning-gd} shows the comparison of our focal F-GD diversity (Figure~\ref{fig:cifar10-pruning-f-gd-t4}) and the baseline GD diversity (Figure~\ref{fig:cifar10-pruning-gd-t4}) in hierarchical pruning for the ensemble teams of size $S=4$. The black dots mark the ensembles that are pruned out by the hierarchical pruning while the red ones represent the remaining selected ensembles. The reference ensemble accuracy 96.33\% of the entire deep ensemble on CIFAR-10 is marked by the horizontal black dashed line.
Overall, the hierarchical pruning with our focal F-GD diversity metric achieved much better performance than the baseline GD diversity metric. There are two primary reasons behind this observation.
{\it First,} our focal diversity metrics can better capture the failure independence of member networks of a deep ensemble than baseline diversity metrics. Therefore, when pruning out a low focal diversity branch, such as in Figure~\ref{fig:cifar10-hq-a-pruning-gd-t3} with $S=3$, most ensembles of a larger size with low diversity (with low F-GD scores) will also be pruned out in Figure~\ref{fig:cifar10-hq-a-pruning-gd-t4} (same as Figure~\ref{fig:cifar10-pruning-f-gd-t4}) with $S=4$.
{\it Second,} focal diversity metrics are more effectively correlated to ensemble accuracy than baseline diversity metrics. Therefore, by pruning out low diversity ensembles, our focal diversity metrics can successfully identify high performance ensemble teams with high ensemble accuracy and low ensemble execution cost.

\begin{figure}[h!]
\centering
    \subfloat[Focal Diversity (F-GD)]{
    \centering
    \includegraphics[width=0.45\textwidth]{cifar10-pruning-f-gd-team-size-4}
    \label{fig:cifar10-pruning-f-gd-t4}
    }
    \subfloat[Baseline Diversity (GD)]{
    \centering
    \includegraphics[width=0.45\textwidth]{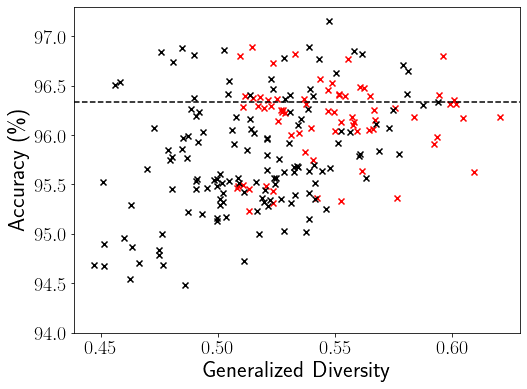}
    \label{fig:cifar10-pruning-gd-t4}
    }
    \caption{Hierarchical Pruning with F-GD and GD for Ensemble Teams of $S=4$ (CIFAR-10, $\beta=10\%, S_d=5$)}
    \label{fig:cifar10-fq-q-pruning-gd}
\end{figure}

\begin{table}[h!]
\centering
\caption{{Comparing Top-3 Ensemble Teams of $S$=4 by F-GD and GD (CIFAR-10, $\beta=10\%, S_d=5$) with the 96.33\% accuracy of the entire ensemble of 10 models shown in Table~\ref{table:example-ensembles-cifar10-imagenet}}} 
\label{table:top-3-ensembles-fgd-gd-cifar10}
\scalebox{1.0}{
\begin{tabular}{|c|c|c|c|c|c|c|}
\hline
Method            & \multicolumn{3}{c|}{F-GD} & \multicolumn{3}{c|}{GD} \\ \hline
Ensemble Team     & 0235    & 0234   & 0123   & 0268   & 0456   & 1345  \\ \hline
Ensemble Acc (\%) & 96.89   & 96.84  & 97.15  & 96.18  & 95.63  & 96.17 \\ \hline
\begin{tabular}[c]{@{}c@{}}Acc Improv (\%)\\ (Over 96.33\%)\end{tabular} & \textbf{0.56}    & \textbf{0.51}   & \textbf{0.82}   & -0.15  & -0.70  & -0.16 \\ \hline\end{tabular}
}
\end{table}

\begin{table}[h!]
\centering
\caption{{Comparing Top-3 Ensemble Teams of $S$=5 by F-GD and GD (ImageNet, $\beta=10\%, S_d=5$) with the 96.33\% accuracy of the entire ensemble of 10 models shown in Table~\ref{table:example-ensembles-cifar10-imagenet}}} 
\label{table:top-3-ensembles-fgd-gd-imagenet}
\scalebox{1.0}{
\begin{tabular}{|c|c|c|c|c|c|c|}
\hline
Method                         & \multicolumn{3}{c|}{F-GD} & \multicolumn{3}{c|}{GD} \\ \hline
Ensemble Team                  & 12345  & 23459  & 23458  & 03478  & 02356 & 01257 \\ \hline
Ensemble Acc (\%)              & 80.77  & 80.50  & 80.44  & 78.53  & 79.25 & 79.19 \\ \hline
\begin{tabular}[c]{@{}c@{}}Acc Improv (\%)\\ (Over 79.82\%)\end{tabular} & \textbf{0.95}   & \textbf{0.68}   & \textbf{0.62}   & -1.29  & -0.57 & -0.63 \\ \hline
\end{tabular}}
\end{table}

Table~\ref{table:top-3-ensembles-fgd-gd-cifar10} lists the Top-3 sub-ensembles selected by F-GD and GD in Figure~\ref{fig:cifar10-fq-q-pruning-gd}. All the Top-3 sub-ensembles selected by our hierarchical pruning using focal diversity F-GD improve the reference accuracy of 96.33\% achieved by the entire ensemble on CIFAR-10. In comparison, the Top-3 sub-ensembles chosen by the baseline GD diversity all result in ensemble accuracy lower than 96.33\%, the reference accuracy of the entire ensemble on CIFAR-10, failing to meet the ensemble pruning objective. Similar observations can be found for the ImageNet dataset in Table~\ref{table:top-3-ensembles-fgd-gd-imagenet}.

The empirical results on all four benchmark datasets demonstrate that our focal diversity based hierarchical pruning framework and algorithms are effective in identifying and selecting space and time efficient sub-ensembles from a given large ensemble team, while offering competitive ensemble accuracy.

\begin{table}[h!]
\centering
\caption{Execution Time (s) Comparison of Baseline and Hierarchical Ensemble Pruning}
\label{table:pruning-execution-time}
\scalebox{1.0}{
\begin{tabular}{|c|c|c|c|}
\hline
\multicolumn{2}{|c|}{Execution Time (s)}       & CIFAR-10 & ImageNet \\ \hline

\multirow{4}{*}{Baseline (No Pruning)}& BD   & 0.90     & 79.08    \\ \cline{2-4}
                                      & GD   & 0.84     & 75.83    \\ \cline{2-4}
                                      & F-BD & 2.01     & 97.26    \\ \cline{2-4}
                                      & F-GD & 2.01     & 93.51    \\ \hline
\multirow{2}{*}{Hierarchical Pruning} & F-BD & \textbf{0.63}     & \textbf{21.89}    \\ \cline{2-4}
                                      & F-GD & \textbf{0.62}     & \textbf{21.79}   \\ \hline
\end{tabular}
}
\end{table}

\subsection{Execution Time Comparison}

We show the execution time of our hierarchical ensemble pruning based on F-BD and F-GD metrics in Table~\ref{table:pruning-execution-time} for CIFAR-10 and ImageNet. The baseline is no pruning, where a diversity score, such as GD and F-GD, is computed for each possible ensemble team. We highlight three interesting observations. 
{\it First,} compared to the baseline focal diversity metrics, F-BD and F-GD, our hierarchical pruning significantly accelerates the process of identifying high-quality sub-ensembles by 3.2$\sim$4.4$\times$, substantially improving the ensemble pruning efficiency.
Table~\ref{table:pruning-execution-time-breakdown} presents the execution time breakdown for each team size $S$ from 2 to 5 using the baseline no pruning and hierarchical ensemble pruning. We list the number of all sub-ensemble teams for each team size in the 2nd column (\#All Teams). For the baseline no pruning, we show the F-GD diversity computation time for all sub-ensemble teams in the 3rd column for CIFAR-10 and 8th column for ImageNet. For the hierarchical pruning approach, we list the execution time for calculating F-GD scores (4th column for CIFAR-10 and 9th column for ImageNet), the number of teams that participate in the diversity computation (5th column for CIFAR-10 and 10th column for ImageNet), and the reduction in both execution time and \#Teams (6th \& 7th columns for CIFAR-10 and 11th \& 12th columns for ImageNet). 
Our hierarchical pruning approach will examine all the 45 smallest ensembles without \#Teams reduction for $S$=2. The slight execution time reduction (0.05s, 3.65\%) on ImageNet for $S$=2 can be attributed to the random measurement errors.
For both CIFAR-10 and ImageNet, the execution time reduction and \#Teams reduction by our hierarchical pruning closely match with each other, such as 89.95\% time reduction and 89.29\% of the teams pruned out for $S$=5 on ImageNet. This indicates that the overall performance improvement (3.2$\sim$4.4$\times$) of our hierarchical pruning originates from the pruned teams that will not participate in diversity computation and thus reduce the execution time.
{\it Second,} for the baseline without pruning, our focal diversity metrics, F-BD and F-GD, take longer execution time than BD and GD, which is due to the computation of focal diversity scores involving two steps for computing the focal negative correlation scores and aggregating the $S$ focal negative correlation scores for each sub-ensemble of size $S$.
{\it Third,} compared to the baseline diversity metrics, BD and GD, our F-BD and F-GD powered hierarchical pruning can reduce the execution time by 26\%$\sim$72\%, further demonstrating the efficiency of our hierarchical pruning approach.

\begin{table}[h!]
\centering
\caption{Execution Time (s) Breakdown Comparison of Baseline and Hierarchical Ensemble Pruning (F-GD)}
\label{table:pruning-execution-time-breakdown}
\scalebox{0.7}{
\begin{tabular}{|c|c|c|c|c|c|c|c|c|c|c|c|}
\hline
\multirow{3}{*}{\makecell{Team\\Size}} & Dataset & \multicolumn{5}{c|}{CIFAR-10}     & \multicolumn{5}{c|}{ImageNet}      \\ \cline{2-12}
 &
  \multirow{2}{*}{\makecell{\#All\\Teams}} &
  \makecell{No\\Pruning} &
  \multicolumn{4}{c|}{Hierarchical Pruning} &
  \makecell{No\\Pruning} &
  \multicolumn{4}{c|}{Hierarchical Pruning} \\ \cline{3-12}
 &
   &
  Time (s) &
  Time (s) &
  \#Teams &
  \makecell{Time\\Reduction\\(\%)} &
  \makecell{\#Teams\\Reduction\\(\%)} &
  Time (s) &
  Time (s) &
  \#Teams &
  \makecell{Time\\Reduction\\(\%)} &
  \makecell{\#Teams\\Reduction\\(\%)} \\ \hline
2                          & 45      & 0.06 & 0.06 & 45 & \textbf{0}     & \textbf{0}  & 1.37  & 1.32 & 45 & \textbf{3.65}  & \textbf{0}  \\ \hline
3                          & 120     & 0.26 & 0.18 & 85 & \textbf{30.77} & \textbf{29.17} & 9.17  & 5.82 & 82 & \textbf{36.53} & \textbf{31.67} \\ \hline
4                          & 210     & 0.63 & 0.23 & 86 & \textbf{63.49} & \textbf{59.05} & 28.34 & 9.78 & 78 & \textbf{65.49} & \textbf{62.86} \\ \hline
5                          & 252     & 1.03 & 0.15 & 37 & \textbf{85.44} & \textbf{85.32} & 52.94 & 5.32 & 27 & \textbf{89.95} & \textbf{89.29} \\ \hline
\end{tabular}
}
\end{table}

\begin{table}[h!]
\centering
\caption{Ensemble Execution Cost Reduction by Hierarchical Ensemble Pruning (ImageNet, $\beta$ = 5\%, $S_d$ = 5)}
\label{table:ensemble-execution-cost-reduction}
\scalebox{0.88}{
\begin{tabular}{|c|c|c|c|c|c|c|}
\hline
\multicolumn{3}{|c|}{\multirow{2}{*}{Ensemble Team}}       & \multirow{2}{*}{Ensemble Acc (\%)} & \multicolumn{3}{c|}{Execution Costs}            \\ \cline{5-7}
\multicolumn{3}{|c|}{}                      &               & \#Params (M)   & GFLOPs        & Inference Time (ms) \\ \hline
\multicolumn{2}{|c|}{Baseline} & 0123456789 & 79.82         & 481.73         & 149.94        & 456.00              \\ \hline
\multirow{3}{*}{\makecell{Top-1 teams\\for each size $S$}} & $S$=3 & 245 & 80.42 (+0.60)                      & 92.64 (\textbf{-81\%}) & 58.50 (\textbf{-61\%}) & 153.69 (\textbf{-66\%}) \\ \cline{2-7}
             & $S$=4           & 2345       & 80.70 (+0.88) & 117.67 (-76\%) & 67.04 (-55\%) & 160.93 (65\%)       \\ \cline{2-7}
             & $S$=5           & 12345      & \textbf{80.77 (+0.95)} & 125.65 (-74\%) & 72.80 (-51\%) & 208.49 (-54\%)     \\ \hline
\end{tabular}
} 
\end{table}

Compared with the entire ensemble of all 10 base models, our focal diversity based hierarchical pruning approach can effectively reduce ensemble execution costs in terms of \#Parameters (storage cost), GFLOPs (computing cost), and inference latency. Table~\ref{table:ensemble-execution-cost-reduction} presents the ensemble execution costs for the entire ensemble and 3 Top-1 ensembles for $S = 3, 4, 5$ identified by our F-GD powered hierarchical pruning approach.
These three sub-ensembles all achieve higher ensemble accuracy than the large 10-model ensemble with much smaller team sizes, significantly improving the ensemble execution efficiency, i.e., cutting down the storage cost by 74\%$\sim$81\%, the computing cost by 51\%$\sim$61\%, and inference latency by 54\%$\sim$66\%.

\section{Conclusion}
This paper presents our hierarchical ensemble pruning approach powered by  the focal diversity metrics. The hierarchical pruning can effectively identify high quality sub-ensembles with a significantly smaller team size and the same or better ensemble accuracy than the entire ensemble team. We made three original contributions.
{\it First,} we optimize ensemble diversity measurements by using focal diversity metrics to accurately capture the failure independence among member networks of a deep ensemble, which closely correlates to ensemble predictive performance and provides effective guidance in ensemble pruning.
{\it Second,} we introduce a hierarchical ensemble pruning algorithm, powered by our focal diversity metrics, which iteratively identifies high quality sub-ensembles and preemptively prunes out the member networks of low diversity.
{\it Third,} we provide a systematic ensemble pruning approach (HQ), which consists of the focal ensemble diversity metric, hierarchical ensemble pruning algorithm and focal diversity consensus voting method.
We conducted comprehensive experimental evaluations on four benchmark datasets, CIFAR-10, ImageNet, Cora and MNIST, which demonstrates that our HQ approach can efficiently prune large ensemble teams and obtain high quality sub-ensembles with enhanced ensemble accuracy and significantly reduced execution cost over the entire large ensemble team.
Deep ensembles have been widely used in many domain-specific applications, including intelligent medical care~\cite{health-online-model-ensemble-serving-icu,dnn-ensemble-covid-19-ct}, intelligent transportation~\cite{rnn-based-path-prediction-deep-ensemble,attention-based-deep-ensemble-taxi-hailing}, and intelligent manufacturing~\cite{effective-automatic-TFT-LCD-defect-classification-ensemble,ensemble-cnn-wafer-bin-map-pattern-classification}. One of our future works will focus on applying and optimizing our hierarchical ensemble pruning to enhance the deep ensemble efficiency for these domain-specific applications.

\begin{acks}
We acknowledge the partial support from the NSF CISE grants 1564097, 2038029, 2302720, and 2312758, an IBM Faculty Award, and a CISCO Edge AI grant.
\end{acks}

\bibliographystyle{ACM-Reference-Format}
\bibliography{reference}

\end{document}